%% file: main.tex
\documentclass[10pt,twocolumn,letterpaper]{article}

\usepackage{cvpr}              
\usepackage{url}
\usepackage{makecell}
\usepackage[accsupp]{axessibility} 

\newcommand{\diff}{\mathrm{d}}

\newcommand{\Expectation}{\mathbb{E}}

\newcommand{\btheta}{\boldsymbol{\theta}}

\newcommand{\loss}{\mathcal{L}}

\usepackage{algcompatible}

\input{preamble}

\definecolor{cvprblue}{rgb}{0.21,0.49,0.74}
\usepackage[pagebackref,breaklinks,colorlinks,citecolor=cvprblue]{hyperref}

\def\paperID{9036} 
\def\confName{CVPR}
\def\confYear{2024}

\title{Accelerating Diffusion Sampling with Optimized Time Steps}

\author{Shuchen Xue$^1$, Zhaoqiang Liu$^2$\thanks{Corresponding author.}, Fei Chen$^3$, Shifeng Zhang$^3$, Tianyang Hu$^3$, Enze Xie$^3$, Zhenguo Li$^3$\\ 
$^1$ Academy of Mathematics and Systems Science, Chinese Academy of Sciences \\
$^2$ University of Electronic Science and Technology of China \\
$^3$ Huawei Noah's Ark Lab}

\begin{document}

\maketitle
\input{sec/0_abstract}    
\input{sec/1_intro}

\input{sec/3_related}
\input{sec/2_prelim}
\input{sec/4_prob_form}
\input{sec/5_analysis_method}
\input{sec/6_experiments}

\input{sec/7_conclusion}
{
    \small
    \bibliographystyle{ieeenat_fullname}
    \bibliography{main}
}

\input{sec/X_suppl}

\end{document}

%% file: preamble.tex
\usepackage[mathscr]{eucal}
\usepackage{epsfig,epsf,psfrag}
\usepackage{amssymb,amsmath,amsfonts,latexsym}
\usepackage{amsmath,graphicx}
\usepackage{amsthm}
\usepackage{fixltx2e}
\usepackage{array}
\usepackage{verbatim}
\usepackage{bm}
\usepackage{algpseudocode}
\usepackage{algorithm}
\usepackage{verbatim}
\usepackage{textcomp}
\usepackage{mathrsfs}
\usepackage{epstopdf}
\usepackage{relsize}

\usepackage[dvipsnames]{xcolor}

\newcommand{\nn}{\nonumber}

\newcommand{\calN}{\mathcal{N}}

\newcommand{\calP}{\mathcal{P}}

\newcommand{\bI}{\mathbf{I}}

\newcommand{\bs}{\mathbf{s}}

\newcommand{\bw}{\mathbf{w}}

\newcommand{\bx}{\mathbf{x}}

\newcommand{\rmd}{\mathrm{d}}

\newcommand{\bbE}{\mathbb{E}}

\newcommand{\bbP}{\mathbb{P}}

\newcommand{\bbR}{\mathbb{R}}

\DeclareMathAlphabet{\mathbsf}{OT1}{cmss}{bx}{n}
\DeclareMathAlphabet{\mathssf}{OT1}{cmss}{m}{sl}

\DeclareSymbolFont{bsfletters}{OT1}{cmss}{bx}{n}  
\DeclareSymbolFont{ssfletters}{OT1}{cmss}{m}{n}
\DeclareMathSymbol{\bsfGamma}{0}{bsfletters}{'000}
\DeclareMathSymbol{\ssfGamma}{0}{ssfletters}{'000}
\DeclareMathSymbol{\bsfDelta}{0}{bsfletters}{'001}
\DeclareMathSymbol{\ssfDelta}{0}{ssfletters}{'001}
\DeclareMathSymbol{\bsfTheta}{0}{bsfletters}{'002}
\DeclareMathSymbol{\ssfTheta}{0}{ssfletters}{'002}
\DeclareMathSymbol{\bsfLambda}{0}{bsfletters}{'003}
\DeclareMathSymbol{\ssfLambda}{0}{ssfletters}{'003}
\DeclareMathSymbol{\bsfXi}{0}{bsfletters}{'004}
\DeclareMathSymbol{\ssfXi}{0}{ssfletters}{'004}
\DeclareMathSymbol{\bsfPi}{0}{bsfletters}{'005}
\DeclareMathSymbol{\ssfPi}{0}{ssfletters}{'005}
\DeclareMathSymbol{\bsfSigma}{0}{bsfletters}{'006}
\DeclareMathSymbol{\ssfSigma}{0}{ssfletters}{'006}
\DeclareMathSymbol{\bsfUpsilon}{0}{bsfletters}{'007}
\DeclareMathSymbol{\ssfUpsilon}{0}{ssfletters}{'007}
\DeclareMathSymbol{\bsfPhi}{0}{bsfletters}{'010}
\DeclareMathSymbol{\ssfPhi}{0}{ssfletters}{'010}
\DeclareMathSymbol{\bsfPsi}{0}{bsfletters}{'011}
\DeclareMathSymbol{\ssfPsi}{0}{ssfletters}{'011}
\DeclareMathSymbol{\bsfOmega}{0}{bsfletters}{'012}
\DeclareMathSymbol{\ssfOmega}{0}{ssfletters}{'012}

\newcommand{\bepsilon}{\bm{\epsilon}}

\newcommand{\bxi}{\bm{\xi}}

\newtheorem{theorem}{Theorem} 
\newtheorem{lemma}{Lemma}

\newtheorem{assumption}{Assumption}

\newtheorem{remark}{Remark}

\newcommand{\qednew}{\nobreak \ifvmode \relax \else
      \ifdim\lastskip<1.5em \hskip-\lastskip
      \hskip1.5em plus0em minus0.5em \fi \nobreak
      \vrule height0.75em width0.5em depth0.25em\fi}



%% file: sec/0_abstract.tex
\begin{abstract}
Diffusion probabilistic models (DPMs) have shown remarkable performance in high-resolution image synthesis, but their sampling efficiency is still to be desired due to the typically large number of sampling steps. Recent advancements in high-order numerical ODE solvers for DPMs have enabled the generation of high-quality images with much fewer sampling steps. While this is a significant development, most sampling methods still employ uniform time steps, which is not optimal when using a small number of steps. To address this issue, we propose a general framework for designing an optimization problem that seeks more appropriate time steps for a specific numerical ODE solver for DPMs. This optimization problem aims to minimize the distance between the ground-truth solution to the ODE and an approximate solution corresponding to the numerical solver. It can be efficiently solved using the constrained trust region method, taking less than $15$ seconds. Our extensive experiments on both unconditional and conditional sampling using pixel- and latent-space DPMs demonstrate that, when combined with the state-of-the-art sampling method UniPC, our optimized time steps significantly improve image generation performance in terms of FID scores for datasets such as CIFAR-10 and ImageNet, compared to using uniform time steps.\footnote{The code is available at \url{https://github.com/scxue/DM-NonUniform}}
\end{abstract}

%% file: sec/1_intro.tex
\section{Introduction}
\label{sec:intro}

Diffusion Probabilistic Models (DPMs), as referenced in various studies \citep{sohl2015deep, ho2020denoising, song2020score}, have garnered significant attention for their success in a wide range of generative tasks. These tasks include, but are not limited to, image synthesis \citep{dhariwal2021diffusion, Rombach_2022_CVPR, ho2022cascaded}, video generation \citep{ho2022vdm,blattmann2023videoldm}, text-to-image conversion \citep{2021glide,dalle2,imagen}, and speech synthesis \citep{lam2022bddm, BinauralGrad}. Central to the operation of DPMs is a dual-process mechanism: a forward diffusion process incrementally introduces noise into the data, while a reverse diffusion process is employed to reconstruct data from this noise. Despite their superior performance in generation tasks when compared to alternative methodologies like Generative Adversarial Networks (GANs) \citep{goodfellow2014generative} or Variational Autoencoders (VAEs) \citep{kingma2013auto}, DPMs require a considerable number of network function evaluations (NFEs) \citep{ho2020denoising}, which is a notable computational limitation for their broader application.

The scholarly pursuit to enhance the sampling efficiency of DPMs bifurcates into two distinct categories: methods that require additional training of the DPMs~\citep{luhman2021knowledge,salimans2022progressive,meng2022on,watson2021learning,xiao2021tackling, wang2023diffusiongan, song2023consistency, luo2023diffinstruct} and those that do not~\citep{song2020denoising,jolicoeur2021gotta,bao2022analytic,Karras2022edm,lu2022dpmsa,zhao2023unipc,xue2023sasolver}. This paper focuses on the latter training-free methods, aiming to improve the sampling efficiency of DPMs. Contemporary training-free samplers utilize efficient numerical schemes to solve the diffusion Stochastic Differential Equations (SDEs) or Ordinary Differential Equations (ODEs). 
While introducing stochasticity in diffusion sampling has been shown to achieve better quality and diversity \citep{xue2023sasolver}, ODE-based sampling methods are superior when the sampling steps are fewer. 
For example, \citet{song2020denoising} provide an empirically efficient solver DDIM. \citet{zhang2023fast} and \citet{lu2022dpmsa} point out the semi-linear structure of diffusion ODEs, and develop higher-order ODE samplers based on it. \citet{zhao2023unipc} further improve these samplers in terms of NFEs by integrating the mechanism of the predictor-corrector method. 
Due to its better efficiency, sampling from diffusion ODE has been predominantly adopted. 
Quality samples can be generated within as few as 10 steps. 
However, state-of-the-art ODE-based samplers are still far from optimal. How to further improve efficiency is an active field of research.  


In this paper, we focus on another important but less explored aspect of diffusion sampling, the \textit{discretization scheme for time steps (sampling schedule)}. Our main contributions are summarized as follows:
\begin{itemize}
    \item Motivated by classic assumptions on score approximation adopted in relevant theoretical works on DPMs, we propose a general framework for constructing an optimization problem with respect to time steps, in order to minimize the distance between the ground-truth solution and an approximate solution to the corresponding ODE. The parameters for the optimization problem can be easily obtained for any explicit numerical ODE solver. In particular, we allow the numerical ODE solver to have varying orders in different steps. The optimization problem can be efficiently solved using the constrained trust region method, requiring less than $15$ seconds.

    \item We construct the optimization problem with respect to time steps for DPM-Solver++~\cite{lu2023dpmsolver++} and UniPC~\cite{zhao2023unipc},  which are recently proposed high-order numerical solvers for DPMs. Extensive experimental results show that our method consistently improves the performance for both unconditional sampling and conditional sampling in both pixel and latent space diffusion models. For example, when combined with UniPC and only using 5 neural function evaluations (NFEs), the time steps found by our optimization problem lead to an FID of 11.91 for unconditional pixel-space CIFAR-10~\citep{krizhevsky2014cifar}, an FID of 10.47 on conditional pixel-space ImageNet~\citep{deng2009imagenet} 64x64, and an FID of 8.66 on conditional latent-space ImageNet 256x256.
\end{itemize}

%% file: sec/3_related.tex

\section{Related Work}

\paragraph{Training-based Methods} 
Training-based methods, e.g., knowledge distillation~\citep{luhman2021knowledge,salimans2022progressive,meng2022on,luo2023diffinstruct}, learning-to-sample~\citep{watson2021learning},  integration with GANs~\citep{xiao2021tackling, wang2023diffusiongan}, and learning the consistency of diffusion ODE~\citep{song2023consistency}, have the potential to sampling for one or very few steps to enhance the efficiency. However, they lack a plug-and-play nature and require substantial extra training, which greatly limits their applicability across diverse tasks. 
We mainly focus on the training-free methods in this paper. Nevertheless, our investigation for the optimal sampling strategy can be orthogonally combined with training-based methods to further boost the performance.

\paragraph{Adaptive Step Size}
Adaptive step size during diffusion sampling has also been explored. It's widely used in numerical solutions of ordinary differential equations the errors of the method.
\citet{jolicoeur2021gotta} designed an adaptive step size via a lower-order and a higher-order method. \citet{lu2022dpmsa} also proposes an adaptive step schedule using a similar method. \citet{gao2023fast} propose fast sampling through the restricting backward error schedule based on dynamically moderating the long-time backward error. However, the empirical performances of these methods in a few steps can still be improved.

\paragraph{Learning to Schedule}
Most related to our work is a line of research that also aims to find the optimal time schedule. 
\citet{watson2021learning} use dynamic programming to discover the optimal time schedule with maximum Evidence Lower Bound (ELBO). \citet{wang2023learning} leverage reinforcement learning method to search a sampling schedule. \citet{liu2023oms} design a predictor-based search algorithm to optimize both the sampling schedule and decide which model to sample from at each step given a set of pre-trained diffusion models. \citet{xia2023towards} propose to train a timestep aligner to align the sampling schedule. \citet{li2023autodiffusion} propose to utilize the evolution algorithm to search over the space of sampling schedules in terms of FID score.
In comparison, our method has negligible computational cost compared to the learning-based method.

%% file: sec/2_prelim.tex
\section{Preliminary}
\label{sec:formatting}
In this section, we provide introductory discussions about diffusion models and the commonly-used discretization schemes for time steps.
\subsection{Diffusion Models}
In the regime of the continuous SDE, DPMs~\citep{sohl2015deep, ho2020denoising, song2020score, kingma2021variational} construct noisy data through the following linear SDE:
\begin{equation}
\label{eq:forward process}
\rmd \bx_t = f(t) \bx_t \rmd t + g(t)  \rmd \bw_t,
\end{equation}
where $\bw_t \in \bbR^D$ represents the standard Wiener process, and $f(t) \bx_t$ and $g(t)$ denote the drift and diffusion coefficients respectively. In addition, for any $t \in [0,T]$, the distribution of $\bx_t$ conditioned on $\bx_0$ is a Gaussian distribution with mean vector $\alpha(t)\bx_0$ and covariance matrix $\sigma^2(t) \bI$, i.e., $\bx_t | \bx_0 \sim \calN(\alpha(t) \bx_0, \sigma^2(t)\bI)$, where positive functions $\alpha(t)$ and $\sigma(t)$ are differentiable with bounded derivatives, and are denoted as $\alpha_t$ and $\sigma_t$ for brevity. Let $q_t$ denote the marginal distribution of $\bx_t$, and $q_{0t}$ be the distribution of $\bx_t$ conditioned on $\bx_0$. The functions $\alpha_t$ and $\sigma_t$ are chosen such that $q_T$ closely approximate a zero-mean Gaussian distribution with covariance matrix $\tilde{\sigma}^2 \bI$ for some $\tilde{\sigma} > 0$, and the signal-to-noise-ratio (SNR) $\alpha_t^2/\sigma_t^2$ is strictly decreasing with respect to $t$. Moreover, to ensure that the SDE in~\eqref{eq:forward process} hase the same transition distribution $q_{0t}$ for $\bx_t$ conditioned on $\bx_0$, $f(t)$ and $g(t)$ need to be dependent on $\alpha_t$ and $\sigma_t$ in the following form:
\begin{equation}
\label{eq:relation_f_alpha}
f(t) = \frac{\diff \log \alpha_t}{\diff t},\hspace{4mm} g^2(t) = \frac{\diff \sigma_t^2}{\diff t} - 2\frac{\diff \log \alpha_t}{\diff t} \sigma_t^2.
\end{equation}
\citet{ANDERSON1982313} establishes a pivotal theorem that the forward process~\eqref{eq:forward process} has an equivalent reverse-time diffusion process (from $T$ to $0$) as in the following equation, so that the generating process is equivalent to solving the diffusion SDE~\citep{ho2020denoising,song2020score}:
\begin{equation}
\label{eq:reverse SDE}
\rmd \bx_t = \left[f(t) \bx_t - g^2(t)\nabla_{\bx}\log q_t(\bx_t)\right] \rmd t + g(t)  \rmd \bar{\bw}_t,
\end{equation}
where $\bar{\bw}_t$ represents the Wiener process in reverse time, and $\nabla_{\bx}\log q_t(\bx)$ is the score function. 
\par
Moreover, \citet{song2020score} also show that there exists a corresponding deterministic process whose trajectories share the same marginal probability densities $q_t(\bx)$ as those of~\eqref{eq:reverse SDE}, thus laying the foundation of efficient sampling via numerical ODE solvers~\citep{lu2022dpmsa,zhao2023unipc}:
\begin{equation}
\label{eq:reverse ODE}
\diff \bx_t = \left[f(t) \bx_t - \frac{1}{2}g^2(t)\nabla_{\bx}\log q_t(\bx_t)\right] \diff t.
\end{equation}
We usually train a score network $\bs_{\btheta}(\bx, t)$ parameterized by $\btheta$ to approximate the score function $\nabla_{\bx}\log q_t(\bx)$ in~\eqref{eq:reverse SDE} by optimizing the denoising score matching loss~\citep{song2020score}: 
\begin{equation}
\label{eq: score matching loss}
\loss = \Expectation_t \Bigl\{ \omega(t) \Expectation_{\bx_0, \bx_t} \bigl[\bigl\| \bs_{\btheta}(\bx_t, t) - \nabla_{\bx}\log q_{0t}(\bx_t|\bx_0) \bigl\|_2^2 \bigl] \Bigl\},
\end{equation}
where $\omega(t)$ is a weighting function. In practice, two methods are commonly used to reparameterize the score network~\citep{kingma2021variational}. The first approach utilizes a \textit{noise prediction network} $\bepsilon_{\btheta}$ such that $\bepsilon_{\btheta} (\bx, t) = -\sigma_t \bs_{\btheta}(\bx, t)$, and the second approach employs a \textit{data prediction network} $\bx_{\btheta}$ such that $\bx_{\btheta}(\bx,t) = (\sigma_t^2 \bs_{\btheta}(\bx,t) + \bx)/\alpha_t$. In particular, for the data prediction network $\bx_{\btheta}$, based on score approximation and~\eqref{eq:relation_f_alpha}, the reverse ODE in~\eqref{eq:reverse ODE} can be expressed as\footnote{Throughout this work, we focus on the case of using a data prediction network $\bx_{\btheta}$. For the case of using a noise prediction network $\bepsilon_{\btheta}$, we can utilize the formula $\bx_{\btheta}(\bx,t) = \frac{\bx - \sigma_t\bepsilon_{\btheta}(\bx,t)}{\alpha_t}$ to transform it into the data prediction network.}
\begin{equation}
    \frac{\rmd \bx_t}{\rmd t} = \frac{\rmd \log\sigma_t}{\rmd t} \bx_t + \frac{\alpha_t \rmd \log \frac{\alpha_t}{\sigma_t}}{\rmd t} \bx_{\btheta}(\bx_t,t),\label{eq:reverse_ODE_data}
\end{equation}
with the initial vector $\bx_T$ being sampled from $\calN(\bm{0},\tilde{\sigma}^2 \bI)$. 

\subsection{Discretization Schemes}

As mentioned above, we can numerically solve a reverse ODE such as~\eqref{eq:reverse_ODE_data} to perform the task of sampling from diffusion models. Different numerical methods are then employed to approximately solve the ODE, for which the time interval $[\epsilon, T]$ needs to be divided into $N$ sub-intervals for some positive integer $N$ using $N + 1$ time steps $T = t_0 > t_1 > \ldots > t_N = \epsilon$.\footnote{Practically, it is typical to use end time $t = \epsilon >0$ instead of $t = 0$ to avoid numerical issues for small $t$, see~\cite[Appendix~D.1]{lu2022dpmsa}. The parameter $\epsilon$ is assumed to be fixed in our work. For the proper selection of $\epsilon$, the reader may refer to prior works such as~\cite{kim2022soft}.} Let $\lambda_t = \log (\alpha_t/\sigma_t)$ denote the one half of the log-SNR. Popular discretization schemes for the time steps include the i) uniform-$t$ scheme, ii) uniform-$\lambda$ scheme~\cite{lu2022dpmsa}, and the discretization strategy used in~\cite{Karras2022edm}, which we refer to as the EDM scheme. 

More specifically, for the uniform-$t$ scheme, we split the interval $[\epsilon, T]$ uniformly and obtain
\begin{equation}
\label{eq: uniform t}
t_n = T + \frac{n}{N} (\epsilon - T) \text{ for } n = 0,1,\ldots, N. 
\end{equation}
Note that since $\lambda_t$ is a strictly decreasing function of $t$, it has an inverse function $t_{\lambda}(\cdot)$. For the uniform-$\lambda$ scheme, we split the interval $[\lambda_{\epsilon},\lambda_{T}]$ uniformly, and convert the $\lambda$ values back to $t$ values. That is, we have
\begin{equation}
\label{eq: uniform logSNR}
t_n = t_{\lambda}\left(\lambda_T + \frac{n}{N}(\lambda_{\epsilon} - \lambda_T)\right) \text{ for } n = 0,1,\ldots, N. 
\end{equation}
In addition,~\citet{Karras2022edm} propose to do a variable substitution $\kappa_t = \frac{\sigma_t}{\alpha_t}$, which is the reciprocal of the square root of SNR. For the EDM scheme, they propose to uniformly discretize the $\kappa_t^{\frac{1}{\rho}}$ for some positive integer $\rho$, i.e., 
\begin{equation}
\label{eq: edm}
t_n = t_{\kappa}\left(\left(\kappa_T^{\frac{1}{\rho}}+ \frac{n}{N}(\kappa_{\epsilon}^{\frac{1}{\rho}} - \kappa_T^{\frac{1}{\rho}})\right)^\rho\right)  \text{ for } n = 0,1,\ldots, N,
\end{equation}
where $t_{\kappa}(\cdot)$ denotes the inverse function of $\kappa_t$, which is guaranteed to exist as $\kappa_t$ is strictly increasing with respect to $t$. \citet{Karras2022edm} performed some ablation studies and suggested to set the parameter $\rho = 7$ in their experiments.

While the discretization scheme plays an important role in the numerical solvers, the above-mentioned schemes have been widely employed for convenience, leaving a significant room of improvements. 


%% file: sec/4_prob_form.tex
\section{Problem Formulation}
\label{sec:prob_form}

 Recall that $\lambda_t = \log(\alpha_t/\sigma_t)$ is strictly decreasing with respect to $t$, and it has an inverse function $t_\lambda(\cdot)$. Then, the term $\bx_{\btheta}(\bx_t,t)$ in~\eqref{eq:reverse_ODE_data} can be written as $\bx_{\btheta}(\bx_{t_\lambda(\lambda)},t_{\lambda}(\lambda)) = \hat{\bx}_{\btheta}(\hat{\bx}_{\lambda},\lambda)$. Let $\bm{f}(\lambda) = \hat{\bx}_{\btheta}(\hat{\bx}_{\lambda},\lambda)$. For the fixed starting time of sampling $T$ and end time of sampling $\epsilon$, the integral formula~\cite[Eq.~(8)]{lu2023dpmsolver++} gives the exact solution to the ODE in~\eqref{eq:reverse_ODE_data} at time $\epsilon$:
\begin{equation*}
    \bx_\epsilon = \frac{\sigma_\epsilon}{\sigma_T}\bx_T + \sigma_\epsilon \int_{\lambda_T}^{\lambda_\epsilon} e^{\lambda} \bm{f}(\lambda)\rmd \lambda.
\end{equation*}
However, $\bm{f}(\lambda)$ is a complicated neural network function and we cannot take the above integral exactly. Instead, we need to choose a sequence of time steps $T = t_0 > t_{1} > \ldots > t_N =\epsilon$ to split the integral as $N$ parts and appropriately approximate the integral in each part separately. In particular, for any $n \in \{1,2,\ldots,N\}$ and a positive integer $k_n \le n$, a $k_n$-th order explicit method for numerical ODE typically uses a $(k_n-1)$-th degree (local) polynomial $\bm{\calP}_{n; k_n-1}(\lambda)$ that involves the function values of $\bm{f}(\cdot)$ at the points $\lambda_{n-k_n},\lambda_{n-k_n+1},\ldots,\lambda_{n-1}$ to approximate $\bm{f}(\lambda)$ within the interval $[\lambda_{t_{n-1}},\lambda_{t_n}]$. That is, the polynomial $\bm{\calP}_{n; k_n-1}(\lambda)$ can be expressed as
\begin{equation}\label{eq:calP_local}
    \bm{\calP}_{n; k_n-1}(\lambda) = \sum_{j=0}^{k_n-1} \ell_{n;k_n,j}(\lambda) \bm{f}(\lambda_{n-k_n+j}),
\end{equation}
where $\ell_{n;k_n,j}(\lambda)$ are certain $(k_n-1)$-th degree polynomials of $\lambda$. 
\begin{remark}\label{remark:calP_local}
    If setting $\bm{\calP}_{n; k_n-1}(\lambda)$ to be the Lagrange polynomial with interpolation points $\lambda_{n-k_n}, \lambda_{n-k_n+1},\ldots,\lambda_{n-1}$, we obtain 
    \begin{equation}
         \ell_{n;k_n,j}(\lambda) := \prod_{i=0, i\ne j}^{k_n-1} \frac{\lambda - \lambda_{n-k_n+i}}{\lambda_{n-k_n+j}-\lambda_{n-k_n+i}}. \label{eq:lagrange_poly_aux}
    \end{equation}
    We note that such an approximation has been essentially considered in the UniPC method proposed in~\cite{zhao2023unipc}, if not using the corrector and a special trick for the case of $k_n = 2$.\footnote{While the iterative formula in~\cite{zhao2023unipc} appears to be different, we have numerically calculated that~\eqref{eq:lagrange_poly_aux} leads to the same coefficients for the neural function values as those for~\cite{zhao2023unipc}.} In addition, a similar idea was considered in~\cite{zhang2023fast} with using local Lagrange polynomials to approximate $e^{\lambda}\bm{f}(\lambda)$.
    
    Or alternatively, $\bm{\calP}_{n;k_n-1}(\lambda)$ may be set as the polynomial corresponding to the Taylor expansion of $\bm{f}(\lambda)$ at $\lambda_{n-1}$, as was done in the works~\cite{lu2022dpmsa,lu2023dpmsolver++}. That is,  
\begin{align}
    & \bm{\calP}_{n;k_n-1}(\lambda) = \sum_{j=0}^{k_n-1} \bm{f}^{(j)}(\lambda_{n-1}) \frac{(\lambda-\lambda_{n-1})^j}{j!}, \label{eq:taylor_poly}
\end{align}
where $\bm{f}^{(j)}(\lambda_{n-1})$ can be further approximated using the neural function values at $\lambda_{n-k_n}, \lambda_{n-k_n+1},\ldots,\lambda_{n-1}$~\cite{lu2023dpmsolver++}. In particular, if letting $h_n = \lambda_{n+1}-\lambda_n$ for any $n \in [N]$, we have for $j=1$ that $\bm{f}^{(j)}(\lambda_{n-1}) \approx \frac{\bm{f}(\lambda_{n-1}) - \bm{f}(\lambda_{n-2})}{h_{n-2}}$ and for $j=2$ that $\bm{f}^{(j)}(\lambda_{n-1}) \approx \frac{2}{h_{n-2}(h_{n-2}+h_{n-3})}\bm{f}(\lambda_{n-1}) - \frac{2}{h_{n-2}h_{n-3}}\bm{f}(\lambda_{n-2})+ \frac{2}{h_{n-3}(h_{n-2}+h_{n-3})}\bm{f}(\lambda_{n-3})$. Therefore, we observe that Taylor expansion will also give local polynomials $\bm{\calP}_{n;k_n-1}(\lambda)$ that are of the form in~\eqref{eq:calP_local}, although they will be slightly different with those for the case of using Lagrange approximations.
\end{remark}

Based on~\eqref{eq:calP_local}, we obtain an approximation of $\bx_\epsilon$ as follows:
\begin{align}
    &\bx_\epsilon = \frac{\sigma_\epsilon}{\sigma_T}\bx_T + \sigma_\epsilon  \sum_{n=1}^N\int_{\lambda_{t_{n-1}}}^{\lambda_{t_n}} e^{\lambda} \bm{f}(\lambda)\rmd \lambda \\
    & \approx  \frac{\sigma_\epsilon}{\sigma_T}\bx_T + \sigma_\epsilon \sum_{n=1}^N\int_{\lambda_{t_{n-1}}}^{\lambda_{t_n}} e^{\lambda} \bm{\calP}_{n;k_n-1}(\lambda)\rmd \lambda \label{eq:poly_approx_final_m1}\\
    & = \frac{\sigma_\epsilon}{\sigma_T}\bx_T + \sigma_\epsilon \sum_{n=1}^N\int_{\lambda_{t_{n-1}}}^{\lambda_{t_n}} e^{\lambda} \big(\sum_{j=0}^{k_n-1} \ell_{n;k_n,j}(\lambda) \bm{f}(\lambda_{n-k_n+j})\big)\rmd \lambda \\
    & = \frac{\sigma_\epsilon}{\sigma_T}\bx_T + \sigma_\epsilon \sum_{n=1}^N \sum_{j=0}^{k_n-1} w_{n;k_n,j}\bm{f}(\lambda_{n-k_n+j})\label{eq:poly_approx_final_m0}\\
    & := \tilde{\bx}_{\epsilon},\label{eq:poly_approx_final}
\end{align}
where 
\begin{equation}
    w_{n;k_n,j} = \int_{\lambda_{t_{n-1}}}^{\lambda_{t_n}} e^{\lambda}\ell_{n;k_n,j}(\lambda) \rmd \lambda.\label{eq:calc_weights}
\end{equation}
It is worth noting that we would expect that the weights $w_{n;k_n,j}$ satisfy the following 
\begin{equation}
     \sum_{j=0}^{k_n-1} w_{n;k_n,j} = e^{\lambda_n} - e^{\lambda_{n-1}} \label{eq:natural_assump}
\end{equation}
for all $n \in \{1,2,\ldots,N\}$ since for the simplest case that $\bm{f}(\lambda)$ is a vector of all ones, its approximation $\bm{\calP}_{n;k_n-1}(\lambda)$ should also be a vector of all ones, and then~\eqref{eq:natural_assump} follows from~\eqref{eq:poly_approx_final_m1} and~\eqref{eq:poly_approx_final_m0}. The condition in~\eqref{eq:natural_assump} can also be easily verified to hold for the case of using Lagrange and Taylor expansion polynomials discussed in Remark~\ref{remark:calP_local}. 

For the vector $\tilde{\bx}_{\epsilon}$ defined in~\eqref{eq:poly_approx_final}, since our goal is to find an estimated vector that is close to the ground-truth solution $\bx_0$, a natural question is as follows:

\underline{Question}: \textit{Can we design the time steps $t_1,t_2,\ldots,t_{N-1}$ (or equivalently, $\lambda_{t_1},\lambda_{t_2},\ldots,\lambda_{t_{N-1}}$) such that the distance between the ground-truth solution $\bx_0$ and the approximate solution $\tilde{\bx}_\epsilon$ is minimized?}

\begin{remark}
    The definition of $\tilde{\bx}_{\epsilon}$ in~\eqref{eq:poly_approx_final} is fairly general in the sense that for any explicit numerical ODE solver with corresponding local polynomials $\bm{\calP}_{n;k_n-1}(\lambda)$ of the form~\eqref{eq:calP_local}, we can easily calculate all the weights $w_{n;k_n,j}$ via~\eqref{eq:calc_weights}. In particular, we allow the (local) orders $k_n$ to vary with respect to $n$ to encompass general cases for high-order numerical solvers. For example, for a third-order method, we have $k_n = 1$ for $n=1$,  $k_n = 2$ for $n=2$, and $k_n = 3$ for $n \ge 3$. The varying $k_n$ can also handle the customized order schedules tested in~\cite{zhao2023unipc}. 
\end{remark}

%% file: sec/5_analysis_method.tex
\section{Analysis and Method}
\label{sec:ana_method}

To partially answer the question presented in Section~\ref{sec:prob_form}, we follow relevant theoretical works~\cite{de2022convergence,lee2022convergence,chen2023improved,chen2023probability,chen2023sampling,chen2023score,lee2023convergence,pedrotti2023improved} to make the following assumption on the score approximation. 

\begin{assumption}\label{assump:score_approx}
    For any $t \in \{t_0, t_1,\ldots,t_N\}$, the error in the score estimate is bounded in $L^2(q_t)$:
    \begin{align*}
        & \|\nabla_{\bx} \log q_t - \bs_{\btheta}(\cdot,t)\|_{L^2(q_t)}^2 \\
        &= \bbE_{q_t}\left[\|\nabla_{\bx} \log q_t(\bx) - \bs_{\btheta}(\bx,t)\|^2\right]  \le \eta^2 \varepsilon_t^2,
    \end{align*}
    where $\eta > 0$ is an absolute constant. 
\end{assumption}
More specifically, $\varepsilon_t$ can be set to be $1$ in the corresponding assumptions in~\cite{lee2022convergence,chen2023sampling,chen2023probability}, and $\varepsilon_t$ can be set to be $\frac{1}{\sigma_t^2}$ in~\cite[Assumption~1]{lee2023convergence}. Based on Assumption~\ref{assump:score_approx} and the transform between score network $\bs_{\btheta}$ and data prediction network $\bx_{\btheta}$~\cite{kingma2021variational}, we obtain the following lemma. 
\begin{lemma}\label{lem:data_network_approx}
    For any $\bx_0 \sim q_0$ and $P_0 \in (0,1)$, with probability at least $1-P_0$, the following event occurs: For all $t \in \{t_0, t_1,\ldots,t_N\}$ and $\bx_t \sim q_t$, we have 
    \begin{equation*}
        \|\bx_{\btheta}(\bx_t,t) - \bx_0\| \le \tilde{\eta} \tilde{\varepsilon}_t, 
    \end{equation*}
    where $\tilde{\eta} := \sqrt{\frac{N+1}{P_0}} \eta$ and $\tilde{\varepsilon}_t := \frac{\varepsilon_t\sigma_t^2}{\alpha_t}$.
\end{lemma}
\begin{proof}
    For any $u > 0$ and any $t \in \{t_0, t_1,\ldots,t_N\}$, by the Markov inequality, we obtain
    \begin{align}
      &\bbP(\|\nabla_{\bx} \log q_t(\bx_t) - \bs_{\btheta}(\bx_t,t)\| > u) \nn \\
      &=    \bbP(\|\nabla_{\bx} \log q_t(\bx_t) - \bs_{\btheta}(\bx_t,t)\|^2 > u^2) \\
      & \le \frac{\bbE_{q_t}\left[\|\nabla_{\bx} \log q_t(\bx_t) - \bs_{\btheta}(\bx_t,t)\|^2\right]}{u^2} \\
      & \le \frac{\eta^2 \varepsilon_t^2}{u^2},\label{eq:markov_score_approx}
    \end{align}
    where~\eqref{eq:markov_score_approx} follows from Assumption~\ref{assump:score_approx}. Taking a union bound over all $t \in \{t_0, t_1,\ldots,t_N\}$, we obtain that with probability at least $1-\frac{(N+1)\eta^2 \varepsilon_t^2}{u^2}$, it holds for all  $t \in \{t_0, t_1,\ldots,t_N\}$ that
    \begin{equation}
        \|\nabla_{\bx} \log q_t(\bx_t) - \bs_{\btheta}(\bx_t,t)\| \le u. 
    \end{equation}
    Using the transforms $\nabla_{\bx} \log q_t(\bx_t) = \frac{\alpha_t\bx_0-\bx_t}{\sigma_t^2}$ and $\bs_{\btheta}(\bx_t,t) = \frac{\alpha_t\bx_{\btheta}(\bx_t,t)-\bx_t}{\sigma_t^2}$~\cite{kingma2021variational}, we derive that with probability at least $1-\frac{(N+1)\eta^2 \varepsilon_t^2}{u^2}$, it holds for all  $t \in \{t_0, t_1,\ldots,t_N\}$ that
    \begin{equation}
        \|\bx_0 - \bx_{\btheta}(\bx_t,t)\| \le \frac{\sigma_t^2}{\alpha_t} u. 
    \end{equation}
    Setting $u =  \sqrt{\frac{N+1}{P_0}} \eta \varepsilon_t$, 
    we obtain the desired result.
\end{proof}
Note that we have $\tilde{\varepsilon}_t = \frac{\sigma_t^2}{\alpha_t}$ for the case that $\varepsilon_t=1$, and $\tilde{\varepsilon}_t = \frac{1}{\alpha_t}$ for the case that $\varepsilon_t=\frac{1}{\sigma_t^2}$. Additionally, motivated by the training objective for the data prediction network that aims to minimize $\frac{\alpha_t}{\sigma_t}\|\bx_0 -\bx_{\btheta}(\bx_t,t)\|$ for $\bx_0 \sim q_0$ and $\bx_t \sim q_t$~\cite{salimans2022progressive}, it is also natural to consider the case that $\tilde{\varepsilon}_t = \frac{\sigma_t}{\alpha_t}$. Therefore, in our experiments, we will set $\tilde{\varepsilon}_t$ as $\frac{\sigma_t^p}{\alpha_t}$ for some non-negative integer $p$.  

Based on Lemma~\ref{lem:data_network_approx}, we establish the following theorem, which provides an upper bound on the distance between the ground-truth solution $\bx_0$ and the approximate solution $\tilde{\bx}_\epsilon$ defined in~\eqref{eq:poly_approx_final}. 
\begin{theorem}\label{thm:general_k_cases}
    Let the approximate solution $\tilde{\bx}_\epsilon$ be defined as in~\eqref{eq:poly_approx_final} with $\bm{f}(\lambda) = \hat{\bx}_{\btheta}(\hat{\bx}(\lambda),\lambda)$ and the weights $w_{n;k_n,j}$ satisfying~\eqref{eq:natural_assump}. Suppose that Assumption~\ref{assump:score_approx} is satisfied for the score approximation. Then, for any $P_0 \in (0,1)$, we have with probability at least $1-P_0$ that  
\begin{align}
    \|\tilde{\bx}_\epsilon- \bx_0\| &\le \left\|\frac{\sigma_\epsilon}{\sigma_T}\bx_T + \sigma_\epsilon \left(e^{\lambda_\epsilon} - e^{\lambda_T}\right) \bx_0 - \bx_0\right\| \nn \\
    & \quad + \sigma_\epsilon \tilde{\eta} \sum_{i=0}^{N-1} \tilde{\varepsilon}_{t_i}\cdot \Big|\sum_{n-k_n+j=i} w_{n;k_n,j}\Big|, \label{eq:main_ineq_generalk}
\end{align}
where $\tilde{\eta} := \sqrt{\frac{N+1}{P_0}} \eta$ and $\tilde{\varepsilon}_t := \frac{\varepsilon_t\sigma_t^2}{\alpha_t}$. 
\end{theorem}
\begin{proof}
     From Lemma~\ref{lem:data_network_approx}, we have with probability at least $1-P_0$, it holds for all $n \in  \{0,1,\ldots,N\}$ that $\bm{f}(\lambda_{t_n})$ can be written as
     \begin{align}
         \bm{f}(\lambda_{t_n}) &= \hat{\bx}_{\btheta}(\hat{\bx}_{\lambda_{t_n}},\lambda_{t_n})  = \bx_{\btheta}(\bx_{t_n},t_n) \\
         & = \bx_0 + \bxi_{t_n},\label{eq:f_lambda_n_approx}
     \end{align}
     where $\bxi_{t_n}$ satisfies $\|\bxi_{t_n}\| \le \tilde{\eta}\tilde{\varepsilon}_{t_n}$. Then, we obtain
     \begin{align}
         & \tilde{\bx}_{\epsilon} = \frac{\sigma_{\epsilon}}{\sigma_{T}}\bx_{T} + \sigma_{\epsilon} \sum_{n=1}^N \sum_{j=0}^{k_n-1} w_{n;k_n,j} \bm{f}(\lambda_{t_{n-k_n+j}}) \\
         & =  \frac{\sigma_{\epsilon}}{\sigma_{T}}\bx_{T} + \sigma_{\epsilon}  \sum_{n=1}^N \sum_{j=0}^{k_n-1} w_{n;k_n,j}(\bx_0 + \bxi_{t_{n-k_n+j}}) \label{eq:use_f_lambda_n_approx}\\
         & =  \frac{\sigma_{\epsilon}}{\sigma_{T}}\bx_{T} + \sigma_{\epsilon} \sum_{n=1}^N \sum_{j=0}^{k_n-1} w_{n;k_n,j} \bx_0 \nn \\
         & \quad + \sigma_N \sum_{n=1}^N \sum_{j=0}^{k_n-1} w_{n;k_n,j} \bxi_{t_{n-k_n+j}} \\
         & =  \frac{\sigma_{\epsilon}}{\sigma_{T}}\bx_{T} + \sigma_{\epsilon}  \sum_{n=1}^N \left(e^{\lambda_n} - e^{\lambda_{n-1}}\right) \bx_0 \nn \\
          & \quad + \sigma_N \sum_{n=1}^N \sum_{j=0}^{k_n-1} w_{n;k_n,j} \bxi_{t_{n-k_n+j}} \label{eq:use_natural_assump}\\
          & =  \frac{\sigma_{\epsilon}}{\sigma_{T}}\bx_{T} + \sigma_{\epsilon}  (e^{\lambda_\epsilon} -e^{\lambda_T}) \bx_0  + \sigma_N \sum_{i=0}^{N-1} \sum_{n-k_n+j=i}  w_{n;k_n,j} \bxi_{t_{i}},\label{eq:final_error_bound}
     \end{align}
     where~\eqref{eq:use_f_lambda_n_approx} follows from~\eqref{eq:f_lambda_n_approx},~\eqref{eq:use_natural_assump} follows from~\eqref{eq:natural_assump}, and~\eqref{eq:final_error_bound} follows from the condition that $j+1 \le k_n \le n$. Therefore, from the triangle inequality, we have 
     \begin{align}
         &\|\tilde{\bx}_\epsilon - \bx_0\| \le \left\|\frac{\sigma_\epsilon}{\sigma_T}\bx_T + \sigma_\epsilon (e^{\lambda_\epsilon} -e^{\lambda_T}) \bx_0 -\bx_0\right\| \nn\\
         & \quad + \sigma_\epsilon \sum_{i=0}^{N-1} \Big|\sum_{n-k_n+j=i}  w_{n;k_n,j}\Big| \|\bxi_{t_{i}}\| \\
         & \le \left\|\frac{\sigma_\epsilon}{\sigma_T}\bx_T + \sigma_\epsilon (e^{\lambda_\epsilon} -e^{\lambda_T}) \bx_0 -\bx_0\right\| \nn\\
         & \quad + \sigma_\epsilon \tilde{\eta}\sum_{i=0}^{N-1} \tilde{\varepsilon}_{t_i}\cdot\Big|\sum_{n-k_n+j=i}  w_{n;k_n,j}\Big|,
     \end{align}
     which gives the desired upper bound. 
\end{proof}
The starting time of sampling $T$ and the end time of sampling $\epsilon$ are fixed, and we typically have $\sigma_\epsilon \approx 0$ and $\alpha_T \approx 0$. Then, the first term in the upper bound of~\eqref{eq:main_ineq_generalk} will be fixed and small. Note that $\{\lambda_{t_n}\}_{n=0}^{N}$ needs to be a monotonically increasing sequence. Additionally, since $\tilde{\eta}$ is a fixed positive constant, we observe from~\eqref{eq:main_ineq_generalk} that we should choose appropriate $\lambda_{t_1}, \lambda_{t_2},\ldots,\lambda_{t_{N-1}}$ (which lead to the corresponding time steps $t_1,t_2,\ldots,t_{N-1}$) such that $\sum_{i=0}^{N-1} \tilde{\varepsilon}_{t_i}\cdot \big|\sum_{n-k_n+j=i} w_{n;k_n,j}\big|$ is minimized, which gives the following optimization problem:
\begin{align}
   \min_{\lambda_{t_1},\ldots,\lambda_{t_{N-1}}} & \quad \sum_{i=0}^{N-1} \tilde{\varepsilon}_{t_i}\cdot \big|\sum_{n-k_n+j=i} w_{n;k_n,j}\big| \nn \\
    \text{s.t.} & \quad  \lambda_{t_{n+1}} > \lambda_{t_n}, \text{ for } n = 0, 1,\ldots, N-1,\label{eq:steps_opt_problem}
\end{align}
where $\lambda_{t_0} = \lambda_{T}$ and $\lambda_{t_N} = \lambda_{\epsilon}$ are fixed. Note that both $\tilde{\varepsilon}_{t_i}$ and $w_{n;k_n,j}$ are functions dependent on $\lambda_{t_1},\ldots,\lambda_{t_{N-1}}$. More specifically, for any sampling algorithm that uses the local approximation polynomials $\bm{\calP}_{n; k_n-1}(\lambda)$ of the form~\eqref{eq:calP_local}, the weights $w_{n;k_n,j}$ can be easily calculated from~\eqref{eq:calc_weights}. In addition, as mentioned above, we can set $\tilde{\varepsilon}_{t_i}$ to be $\frac{\sigma_{t_i}^p}{\alpha_{t_i}}$ for some non-negative integer $p$,\footnote{We find in the experiments that setting $p=1$ is a good choice for pixel-space generation, and setting $p=2$ is appropriate for latent-space generation.} which can be represented as a functions of $\lambda_{t_i}$. Furthermore, the optimization problem in~\eqref{eq:steps_opt_problem} can be approximately solved highly efficiently using the constrained trust region method.\footnote{Such an algorithm has been implemented in standard Python libraries such as scipy. See~\url{https://docs.scipy.org/doc/scipy/tutorial/optimize.html##defining-linear-constraints}.} 
Our method is summarized in Algorithm \ref{alg:opt_steps}. 

\begin{algorithm}[t]
\caption{\textit{Finding the time steps via~\eqref{eq:steps_opt_problem}}}\label{alg:opt_steps}
{\bf Input}: Number of time steps $N$, initial time of sampling $T$, end time of sampling $\epsilon$, any sampling algorithm that is characterized by local polynomials $\bm{\calP}_{n; k_n-1}(\lambda)$ of the form~\eqref{eq:calP_local}, the function $\tilde{\varepsilon}_t := \frac{\sigma_{t}^p}{\alpha_{t}}$ with a fixed positive integer $p$ for score approximation, the initial values of $\lambda_{t_1},\ldots,\lambda_{t_{N-1}}$ \\ 
\textbf{1}: Set $\lambda_{t_0} = \lambda_{T}$ and $\lambda_{t_N} = \lambda_{\epsilon}$ and calculate $\tilde{\varepsilon}_{t_i}$ for $i=0,1,\ldots,N-1$ \\
\textbf{2}: Calculate the weights $w_{n;k_n,j}$ from~\eqref{eq:calc_weights} \\
\textbf{3}: Solve the optimization problem in~\eqref{eq:steps_opt_problem} via the constrained trust region method\\
{\bf Output}: Optimized $\lambda$ (or equivalently, time) steps $\hat{\lambda}_{t_1},\ldots,\hat{\lambda}_{t_{N-1}}$
\end{algorithm}

%% file: sec/6_experiments.tex
\section{Experiments}
\label{sec:exp}

\begin{table*}[t]
\centering

\begin{tabular}{lcccccccc}
\toprule
Methods \textbackslash NFEs & 5 & 6 & 7 & 8 & 9 & 10 & 12 & 15\\
\midrule
DPM-Solver++ with uniform-$\lambda$ & 29.22 & 13.28 & 7.18 & 5.12 &  4.40 & 4.03 & 3.45 & 3.17 \\
\midrule
DPM-Solver++ with uniform-$t$ & 28.16 & 19.63 & 15.29 & 12.58 &  11.18 & 10.15 & 8.50 & 7.10 \\
\midrule
DPM-Solver++ with EDM & 40.48 & 25.10 & 15.68 & 10.22 &  7.42 & 6.18 & 4.85 & 3.49 \\
\midrule
DPM-Solver++ with optimized step (Ours) & 12.91 & 8.35 & 5.44 & 4.74 &  3.81 & 3.51 & 3.24 & 3.15 \\
\midrule
\midrule
UniPC with uniform-$\lambda$ & 23.22 & 10.33 & 6.18 & 4.80 &  4.19 & 3.87 & 3.34 & 3.17 \\
\midrule
UniPC with uniform-$t$ & 25.11 & 17.40 & 13.54 & 11.33 &  9.83 & 8.89 & 7.38 & 6.18 \\
\midrule
UniPC with EDM & 38.24 & 23.79 & 14.62 & 8.95 &  6.60 & 5.59 & 4.18 & 3.16 \\
\midrule
UniPC with optimized step (Ours) & \bf12.11 & \bf7.23 & \bf4.96 & \bf4.46 &  \bf3.75 & \bf3.50 & \bf3.19 & \bf3.13 \\
\bottomrule
\end{tabular}
\caption{
Sampling quality measured by FID ($\downarrow$) of different discretization schemes of time steps for DPM-Solver++~\citep{lu2023dpmsolver++} and UniPC~\citep{zhao2023unipc} with varying NFEs on CIFAR-10.
}
\vspace{-0.05in}
\label{tab: dpms and unipc}
\end{table*}

\begin{figure*}[th]
\centering
\includegraphics[width=\textwidth]{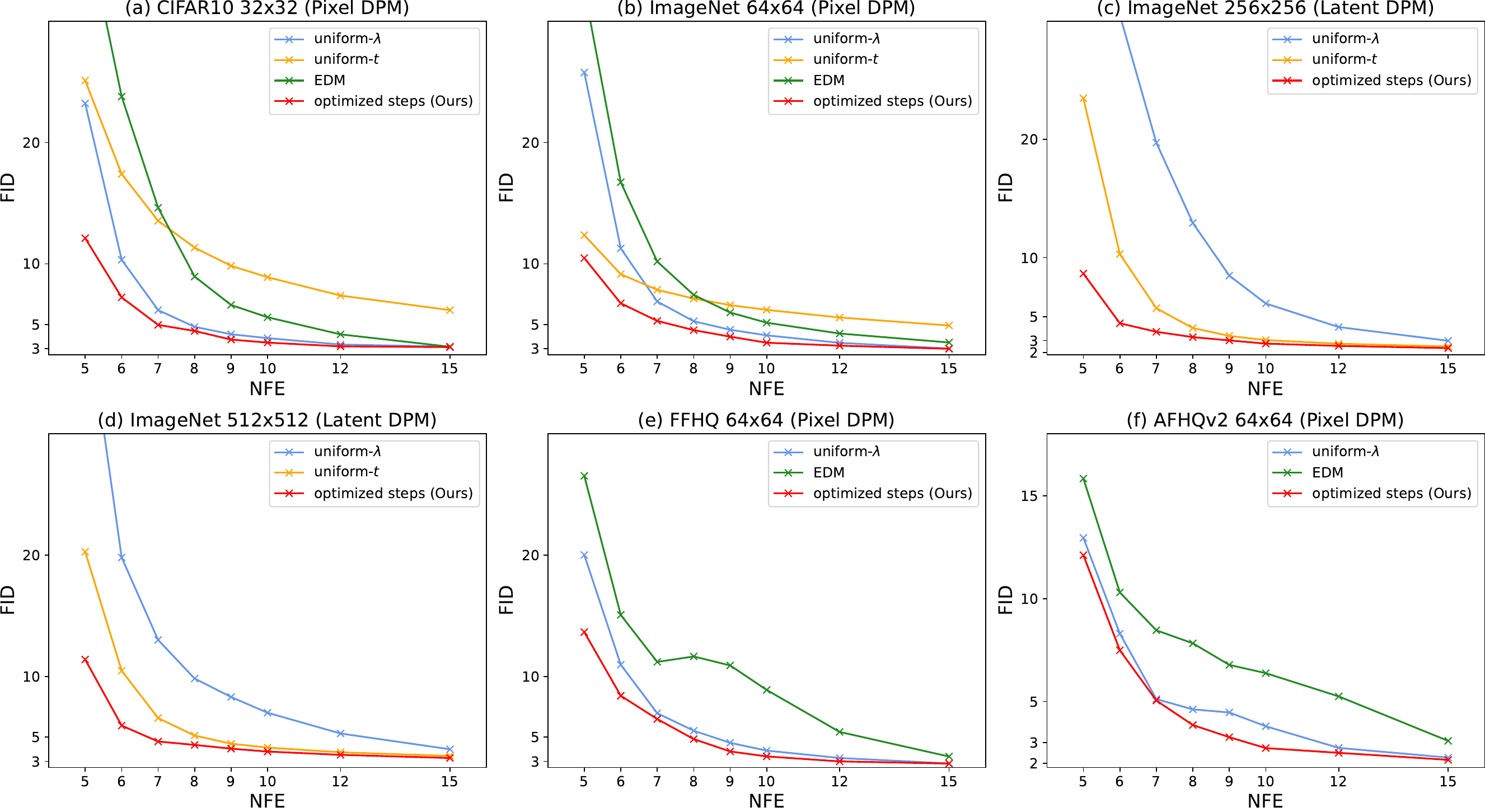}
\caption{Sampling quality measured by FID ($\downarrow$) of different discretization schemes of time steps for UniPC~\citep{zhao2023unipc} with varying NFEs on various DPMs and datasets.}
\label{fig: results}
\end{figure*}

In this section, we demonstrate the effectiveness of our method combined with the currently most popular ODE solvers DPM-Solver++~\citep{lu2023dpmsolver++} and UniPC~\citep{zhao2023unipc}~\footnote{We observe from the experiments that our optimized time steps work effectively for UniPC, although the optimization problem does not take the corrector of UniPC into account, which corresponds to implicit numerical ODE solvers.} for the case of using $5\sim 15$ neural function evaluations (NFEs) with extensive experiments on various datasets. The order is $3$ for both DPM-Solver++ and UniPC. We use Fenchel Inception Distance (FID)~\citep{heusel2017gans} as the evaluation metric to show the effectiveness of our method. Unless otherwise specified, 50K images are sampled for evaluation. The experiments are conducted on a wide range of datasets, with image sizes ranging from 32x32 to 512x512. We also evaluate the performance of various previous state-of-the-art pre-trained diffusion models, including Score-SDE~\citep{song2020score}, ADM~\citep{dhariwal2021diffusion}, EDM~\citep{Karras2022edm}, and DiT~\citep{peebles2022scalable}. which cover pixel-space and latent-space diffusion models, with unconditional generation, conditional generation and generation with classifier-free guidance settings. 

\begin{figure*}[!th]
    \centering
    \setlength{\tabcolsep}{1pt}
    \begin{tabular}{c c c c c}
    \textbf{NFE = 5} & \textbf{Optimized Steps (Ours)} & Uniform-$t$ & EDM & Uniform-$\lambda$\\
     & \textbf{FID = 8.66} & FID = 23.48 & FID = 45.89 & FID = 41.89\\
    \raisebox{5.0\height}{Class label = “coral reef” (973)} &
    \begin{subfigure}{0.16\textwidth}
        \includegraphics[width=\textwidth]{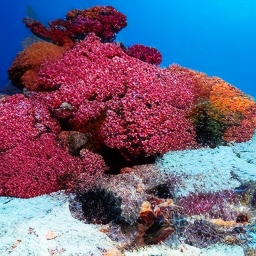}
    \end{subfigure} &
    \begin{subfigure}{0.16\textwidth}
        \includegraphics[width=\textwidth]{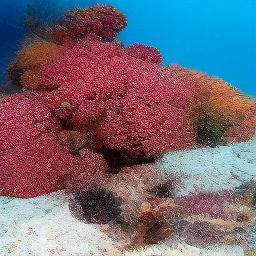}
    \end{subfigure} &
    \begin{subfigure}{0.16\textwidth}
        \includegraphics[width=\textwidth]{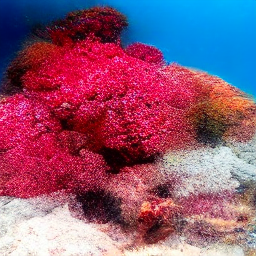}
    \end{subfigure} &
    \begin{subfigure}{0.16\textwidth}
        \includegraphics[width=\textwidth]{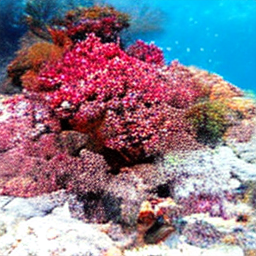}
    \end{subfigure} \\
    \raisebox{5.0\height}{Class label = “golden retriever” (207)} &
    \begin{subfigure}{0.16\textwidth}
        \includegraphics[width=\textwidth]{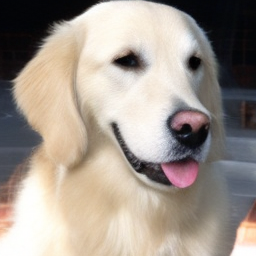}
    \end{subfigure} &
    \begin{subfigure}{0.16\textwidth}
        \includegraphics[width=\textwidth]{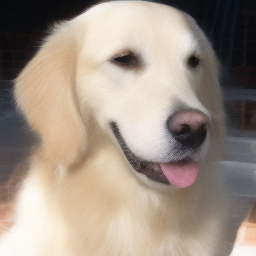}
    \end{subfigure} &
    \begin{subfigure}{0.16\textwidth}
        \includegraphics[width=\textwidth]{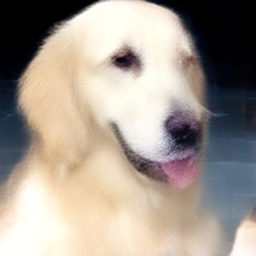}
    \end{subfigure} &
    \begin{subfigure}{0.16\textwidth}
        \includegraphics[width=\textwidth]{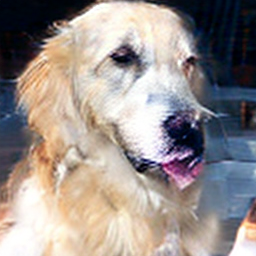}
    \end{subfigure} \\
    \raisebox{5.0\height}{Class label = “lion” (291)} &
    \begin{subfigure}{0.16\textwidth}
        \includegraphics[width=\textwidth]{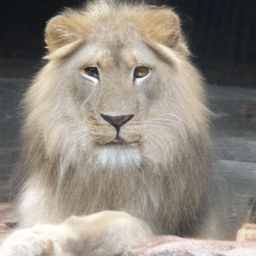}
    \end{subfigure} &
    \begin{subfigure}{0.16\textwidth}
        \includegraphics[width=\textwidth]{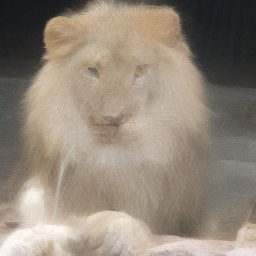}
    \end{subfigure} &
    \begin{subfigure}{0.16\textwidth}
        \includegraphics[width=\textwidth]{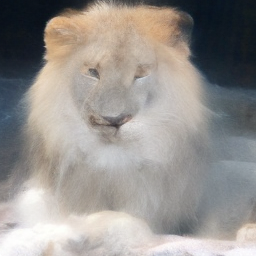}
    \end{subfigure} &
    \begin{subfigure}{0.16\textwidth}
        \includegraphics[width=\textwidth]{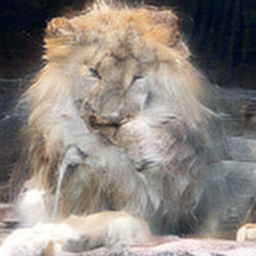}
    \end{subfigure} \\
    \raisebox{5.0\height}{Class label = “lake shore” (975)} &
    \begin{subfigure}{0.16\textwidth}
        \includegraphics[width=\textwidth]{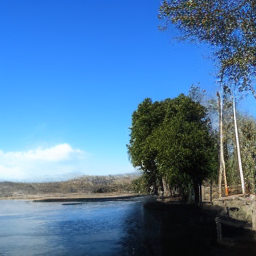}
    \end{subfigure} &
    \begin{subfigure}{0.16\textwidth}
        \includegraphics[width=\textwidth]{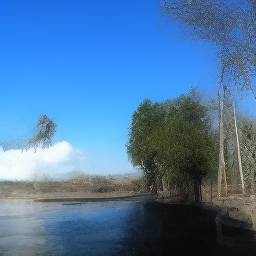}
    \end{subfigure} &
    \begin{subfigure}{0.16\textwidth}
        \includegraphics[width=\textwidth]{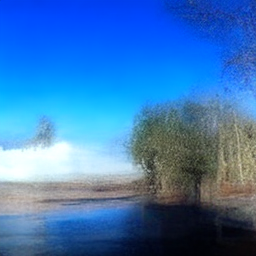}
    \end{subfigure} &
    \begin{subfigure}{0.16\textwidth}
        \includegraphics[width=\textwidth]{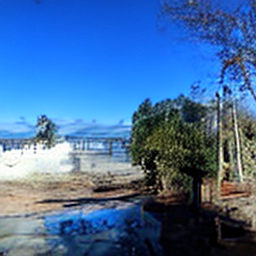}
    \end{subfigure} \\
    \end{tabular}
    \caption{Generated images by UniPC~\citep{zhao2023unipc} with only 5 NFEs for various discretization schemes of time steps from DiT-XL-2 ImageNet 256x256 model~\citep{peebles2022scalable} (with cfg scale $s=1.5$ and the same random seed).}
\label{figure: visual}
\end{figure*}

\subsection{Pixel Diffusion Model Generation}

\paragraph{Results on CIFAR-10 32x32}
For the CIFAR-10 32x32 experiment, We use the \textit{ddpmpp\_deep\_continuous} Score-SDE model from~\citep{song2020score}, which is an unconditional VP-schedule model. The experiment results for UniPC are shown in Fig.~\ref{fig: results}. Our proposed optimized sampling schedule consistently outperforms the other three baseline sampling schedules and achieves state-of-the-art FID results. For example, we improve the FID on CIFAR-10 with only $5$ NFEs from $23.22$ to $12.11$. We also conducted the experiment for DPM-Solver++ on CIFAR-10, see Table~\ref{tab: dpms and unipc}. We found that although our optimized sampling schedule also consistently improves DPM-Solver++, UniPC still performs better no matter which sampling schedule is used. Therefore, in the remaining part of this section, we only report the results for UniPC. Due to the page limit, some additional experimental results have been deferred to the supplementary material.  

\paragraph{Results on ImageNet 64x64}
For the ImageNet 64x64 experiment, we use the \textit{64x64 diffusion} model from ADM~\citep{dhariwal2021diffusion}, which is a conditional VP-schedule model. The experimental results are shown in Fig.\ref{fig: results}. It can be noticed that among the baseline methods, the uniform-$t$ performs the best for NFEs $5\sim 6$, and uniform-$\lambda$ performs best for NFEs $>7$. This indicates that the previous baseline sampling schedules may be far from optimal. Our proposed optimized sampling schedule consistently outperforms the other three baseline sampling schedules and achieves state-of-the-art FID results. For example, we achieved the FID of $10.47$ on ImageNet 64x64 with only $5$ NFEs.

\paragraph{Results on FFHQ 64x64 and AFHQv2 64x64}
For the FFHQ 64x64 and AFHQv2 64x64 experiments, we use the EDM~\citep{Karras2022edm} unconditional model. The experiment results are shown in Fig.~\ref{fig: results}. For the EDM model, the time ranges from $0.002$ to $80$ during sampling rather than $0$ to $1$, which makes the result of uniform-$t$ significantly worse than the other two baseline discretization schemes. Thus we do not include the results of the uniform-$t$ scheme. In comparison, our proposed optimized sampling schedule consistently outperforms the other two baseline sampling schedules and achieves state-of-the-art FID results. For example, we achieved an FID of $13.66$ on FFHQ 64x64 and an FID of $12.11$ on AFHQv2 64x64 with only $5$ NFEs.

\subsection{Latent Diffusion Model Generation}

\paragraph{Results on ImageNet 256x256 and ImageNet 512x512}
We evaluate our method on the ImageNet dataset using the DiT-XL-2~\citep{peebles2022scalable} model, which is a Vision-Transformer-based model in the latent space of KL-8 encoder-decoder. The corresponding classifier-free guidance scale, namely $s=1.5$, is adopted for evaluation. The experiment results are shown in Fig.~\ref{fig: results}. Our proposed optimized sampling schedule consistently outperforms the other two baseline sampling schedules and achieves state-of-the-art FID results. For example, we improve the FID from $23.48$ to $8.66$ on ImageNet 256x256 and from $20.28$ to $11.40$ on ImageNet 512x512 with only $5$ NFEs. Figure \ref{figure: visual} visually demonstrates the effectiveness of our proposed method. Since these experiments are a bit time-consuming, and the EDM discretization scheme for time steps has been shown to not work well for small NFEs, we mainly use uniform-$t$ and uniform-$\lambda$ schemes as baselines.

\subsection{Running Time Analysis}

We test the running time of our Algorithm~\ref{alg:opt_steps} on \textit{Intel(R) Xeon(R) Gold 6278C CPU @ 2.60GHz}. We report the longest running time observed for performing Algorithm~\ref{alg:opt_steps} for $5,6,7,8,9,10,12,15$ NFEs. The experiment results are shown in Table \ref{tab: running time}. Our algorithm can be pre-computed before sampling and the optimization result can be reutilized. The result shows that our optimization problem can be solved within $15$ seconds for NFEs less than or equal to 15, which is negligible. 
In comparison, the learning to schedule methods \citep{watson2021learning, wang2023learning, liu2023oms, xia2023towards, li2023autodiffusion} usually takes several GPU hours for optimization and the overall performance is comparable to ours.

\begin{table}[t]
\centering

\begin{tabular}{lcccccccc}
\toprule
NFEs & 5 & 6 & 7 & 8 & 9 & 10 & 12 & 15\\
\midrule
Time(s) & 1.9 & 2.3 & 5.3 & 5.9 &  7.8 & 8.8 & 11.0 & 14.1 \\
\bottomrule
\end{tabular}
\caption{
Running time of our optimization algorithm.
}
\vspace{-0.05in}
\label{tab: running time}
\end{table}

%% file: sec/7_conclusion.tex
\section{Conclusion}
\label{sec:conc}

In this paper, we propose an optimization-based method to find appropriate time steps to accelerate the sampling of diffusion models, and thus generating high-quality images in a small number of sampling steps. We formulate the problem as a surrogate optimization method with respect to the time steps, which can be efficiently solved via the constrained trust region method. Experimental results on popular image datasets demonstrate that our method can be employed in a plug-and-play manner and achieves state-of-the-art sampling performance based on various pre-trained diffusion models. Our work solves an approximated optimization object which can be further improved if a more accurate formulation can be found. 

%% file: sec/X_suppl.tex
\clearpage
\setcounter{page}{1}
\maketitlesupplementary

\section{Additional Results and Experiment Details}

\subsection{More Completed Results for CIFAR-10, ImageNet 64x64, and ImageNet 256x256}

In this subsection, we present the completed quantitative results for CIFAR-10, ImageNet 64x64, and ImageNet 256x256 for the case of combining our optimized steps with DPM-Solver++ and UniPC. More specifically, more completed results for CIFAR-10  are presented in Table~\ref{tab: dpms and unipc_cifar10}. More completed results for ImageNet 64x64 are presented in Table~\ref{tab: dpms and unipc_imagenet64}. Since the cosine schedule that leads to unbounded $\lambda$ at $T=1.0$ is used for the corresponding codebase, we set $T=0.992$ instead of $T=1.0$. More completed results for ImageNet 256x256  are presented in Table~\ref{tab: dpms and unipc_imagenet256}. 

\begin{table*}[hbt!]
\centering
\begin{tabular}{lcccccccc}
\toprule
Methods \textbackslash NFEs & 5 & 6 & 7 & 8 & 9 & 10 & 12 & 15\\
\midrule
DPM-Solver++ with uniform-$\lambda$ & 29.22 & 13.28 & 7.18 & 5.12 &  4.40 & 4.03 & 3.45 & 3.17 \\
\midrule
DPM-Solver++ with uniform-$t$ & 28.16 & 19.63 & 15.29 & 12.58 &  11.18 & 10.15 & 8.50 & 7.10 \\
\midrule
DPM-Solver++ with EDM & 40.48 & 25.10 & 15.68 & 10.22 &  7.42 & 6.18 & 4.85 & 3.49 \\
\midrule
DPM-Solver++ with uniform-$\lambda$-opt & 12.91 & 8.35 & 5.44 & 4.74 &  3.81 & 3.51 & 3.24 & 3.15 \\
\midrule
DPM-Solver++ with uniform-$t$-opt & 12.67 & 8.13 & 5.63 & 4.98 &  5.47 & 3.66 & 4.63 & 3.16 \\
\midrule
DPM-Solver++ with EDM-opt & 12.93 & 8.04 & 5.90 & 8.63 &  5.14 & 4.72 & 4.12 & 3.16 \\
\midrule
\midrule
UniPC with uniform-$\lambda$ & 23.22 & 10.33 & 6.18 & 4.80 &  4.19 & 3.87 & 3.34 & 3.17 \\
\midrule
UniPC with uniform-$t$ & 25.11 & 17.40 & 13.54 & 11.33 &  9.83 & 8.89 & 7.38 & 6.18 \\
\midrule
UniPC with EDM & 38.24 & 23.79 & 14.62 & 8.95 &  6.60 & 5.59 & 4.18 & 3.16 \\
\midrule
UniPC with uniform-$\lambda$-opt & 12.11 & 7.23 & \bf4.96 & \bf4.46 &  \bf3.75 & 3.50 & \bf3.19 & 3.13 \\
\midrule
UniPC with uniform-$t$-opt & 12.10 & \bf7.01 & 5.27 & 4.53 &  4.69 & \bf3.25 & 3.89 & \bf2.78 \\
\midrule
UniPC with EDM-opt & \bf11.91 & 7.19 & 5.62 & 6.62 &  4.53 & 4.12& 3.63 & 2.87 \\
\bottomrule
\end{tabular}
\caption{Sampling quality measured by FID ($\downarrow$) of different discretization schemes of time steps for DPM-Solver++~\citep{lu2023dpmsolver++} and UniPC~\citep{zhao2023unipc} with varying NFEs on CIFAR-10 (with $T=1.0$ and $\epsilon = 0.0005$).}
\vspace{-0.05in}
\label{tab: dpms and unipc_cifar10}
\end{table*}

\begin{table*}[hbt!]
\centering
\begin{tabular}{lcccccccc}
\toprule
Methods \textbackslash NFEs & 5 & 6 & 7 & 8 & 9 & 10 & 12 & 15\\
\midrule
DPM-Solver++ with uniform-$\lambda$ & 21.78 & 11.08  & 7.07 & 5.38 &  4.57 & 4.18 & 3.67 & 3.23 \\
\midrule
DPM-Solver++ with uniform-$t$ & 20.62 & 14.32 & 10.83 & 8.83 &  7.50 & 6.65 & 5.53 & 4.61 \\
\midrule
DPM-Solver++ with EDM & 25.72 & 15.23 & 9.80 & 7.10 &  5.63 & 4.83 & 3.98 & 3.41 \\
\midrule
DPM-Solver++ with uniform-$\lambda$-opt & 13.98 & 8.28 & 5.66 & 4.46 &  4.20 & 3.81 & 3.56 & 3.22 \\
\midrule
DPM-Solver++ with uniform-$t$-opt & 13.98 & 8.28 & 5.51 & \bf4.39 & \bf3.88 & 3.59 & 3.71 & 3.76 \\
\midrule
DPM-Solver++ with EDM-opt & 13.98 & 8.28 & 5.57 & 4.92 &  5.65 & 4.85 & 3.83 & 3.29 \\
\midrule
\midrule
UniPC with uniform-$\lambda$ & 25.77 & 11.27 & 6.89 & 5.26 &  4.56 & 4.10 & 3.48 & 3.01 \\
\midrule
UniPC with uniform-$t$ & 12.36 & 9.14  & 7.85 & 7.13 &  6.59 & 6.20 & 5.57 & 4.91 \\
\midrule
UniPC with EDM & 32.65 & 16.72 & 10.18 & 7.43 &  5.98 & 5.14 & 4.25 & 3.51 \\
\midrule
UniPC with uniform-$\lambda$-opt & \bf10.47 & \bf6.74 & \bf5.29 & 4.53 &  3.99 & \bf3.49 & \bf3.25 & \bf2.99 \\
\midrule
UniPC with uniform-$t$-opt & \bf10.47 & \bf6.74 & 5.60 & 4.66 &  3.96 & 3.91 & 3.88 & 3.59 \\
\midrule
UniPC with EDM-opt & \bf10.47 & \bf6.74 & 5.39 & 4.98 & 5.02 & 4.30 & 3.53 & 3.18 \\
\bottomrule
\end{tabular}
\caption{Sampling quality measured by FID ($\downarrow$) of different discretization schemes of time steps for DPM-Solver++~\citep{lu2023dpmsolver++} and UniPC~\citep{zhao2023unipc} with varying NFEs on ImageNet 64x64 (with $T=0.992$ and $\epsilon = 0.001$).}
\vspace{-0.05in}
\label{tab: dpms and unipc_imagenet64}
\end{table*}

\begin{table*}[hbt!]
\centering
\begin{tabular}{lcccccccc}
\toprule
Methods \textbackslash NFEs & 5 & 6 & 7 & 8 & 9 & 10 & 12 & 15\\
\midrule
DPM-Solver++ with uniform-$\lambda$ & 38.04 & 20.96  & 14.69 & 11.09 &  8.32 & 6.47 & 4.50 & 3.33 \\
\midrule
DPM-Solver++ with uniform-$t$ & 31.32 & 14.36 & 7.62 & 4.93 &  3.77 & 3.23 & 2.78 & 2.51 \\
\midrule
DPM-Solver++ with EDM & 65.82 & 25.19 & 11.17 & 7.50 &  6.98 & 12.46 & 6.54 & 4.03 \\
\midrule
DPM-Solver++ with uniform-$\lambda$-opt & 12.53 & 5.44 & 3.58 & 7.54 &  5.97 & 4.12 & 3.61 & 3.36 \\
\midrule
DPM-Solver++ with uniform-$t$-opt & 12.53 & 5.44 & 3.89 & 3.81 & 3.13 & 2.79 & \bf2.55 & 2.39 \\
\midrule
DPM-Solver++ with EDM-opt & 12.53 & 5.44 & 3.95 & 3.79 &  3.30 & 3.14 & 2.91 & 2.44 \\
\midrule
\midrule
UniPC with uniform-$\lambda$ & 41.89 & 30.51 & 19.72 &12.94 &  8.49 & 6.13 & 4.14 & 2.98 \\
\midrule
UniPC with uniform-$t$ & 23.48 & 10.31  & 5.73 & 4.06 &  3.39 & 3.04 & 2.73 & 2.50 \\
\midrule
UniPC with EDM & 45.89 & 21.24 & 15.52 & 14.38 &  14.24 & 12.98 & 8.62 & 4.10 \\
\midrule
UniPC with uniform-$\lambda$-opt & \bf8.66 & \bf4.46 & \bf3.57 & 3.72 &  3.40 & 3.01 & 2.94 & 2.53 \\
\midrule
UniPC with uniform-$t$-opt & \bf8.66 & \bf4.46 & 3.74 & \bf3.29 &  \bf3.01 & \bf2.74 & \bf2.55 & \bf2.36 \\
\midrule
UniPC with EDM-opt & \bf8.66 & \bf4.46 & 3.78 & 3.34 & 3.14 & 3.22 & 2.96 & 2.38 \\
\bottomrule
\end{tabular}
\caption{Sampling quality measured by FID ($\downarrow$) of different discretization schemes of time steps for DPM-Solver++~\citep{lu2023dpmsolver++} and UniPC~\citep{zhao2023unipc} with varying NFEs on ImageNet 256x256 (with $T=1.0$ and $\epsilon = 0.001$).}
\vspace{-0.05in}
\label{tab: dpms and unipc_imagenet256}
\end{table*}

\subsection{Additional Quantitative Results}
In this subsection, we also present the quantitative results for FFHQ 64x64 and AFHQv2 64x64, as well as ImageNet 512x512. Since we observed from the quantitative results for CIFAR-10, ImageNet 64x64, and ImageNet 256x256 that for most cases, UniPC outperforms DPM-Solver++, throughout the following, we only present the results for UniPC. As mentioned in Section~\ref{sec:exp}, for the EDM model, the time ranges from $0.002$ to $80$ during sampling rather than $0$ to $1$, which makes the results of the uniform-$t$ scheme significantly worse than those of the other two baseline discretization schemes. Therefore, we do not include the results of the uniform-$t$ scheme for FFHQ 64x64 and AFHQv264x64, for which we use the EDM unconditional model. Additionally, we observed from the results for ImageNet 256x256 that the EDM scheme does not perform well when the number of NFEs is small. As the generation for ImageNet 512x512 is time-consuming, we do not perform experiments for ImageNet 512x512 using the EDM scheme. 

We also observed that for pixel-space generation, the optimized time steps initialized from the uniform-$\lambda$ scheme often lead to the best generation performance, and for latent-space generation, the optimized time steps initialized from the uniform-$t$ scheme often lead to the best generation performance. Therefore, for the pixel-space generation of images for FFHQ 64x64 and AFHQv264x64, we only present the quantitative results for the time steps initialized from the uniform-$\lambda$ scheme. For the latent-space generation of images for ImageNet 512x512, we only present the quantitative results for the time steps initialized from the uniform-$t$ scheme.

The quantitative results for FFHQ 64x64 and AFHQv2 64x64 are presented in Tables~\ref{tab:unipc_ffhq64} and~\ref{tab:unipc_afhq64}, respectively. 

The quantitative results for ImageNet 512x512 are presented in Table~\ref{tab:unipc_imagenet512}.

\begin{figure}[t]
\centering
\includegraphics[width=8cm]{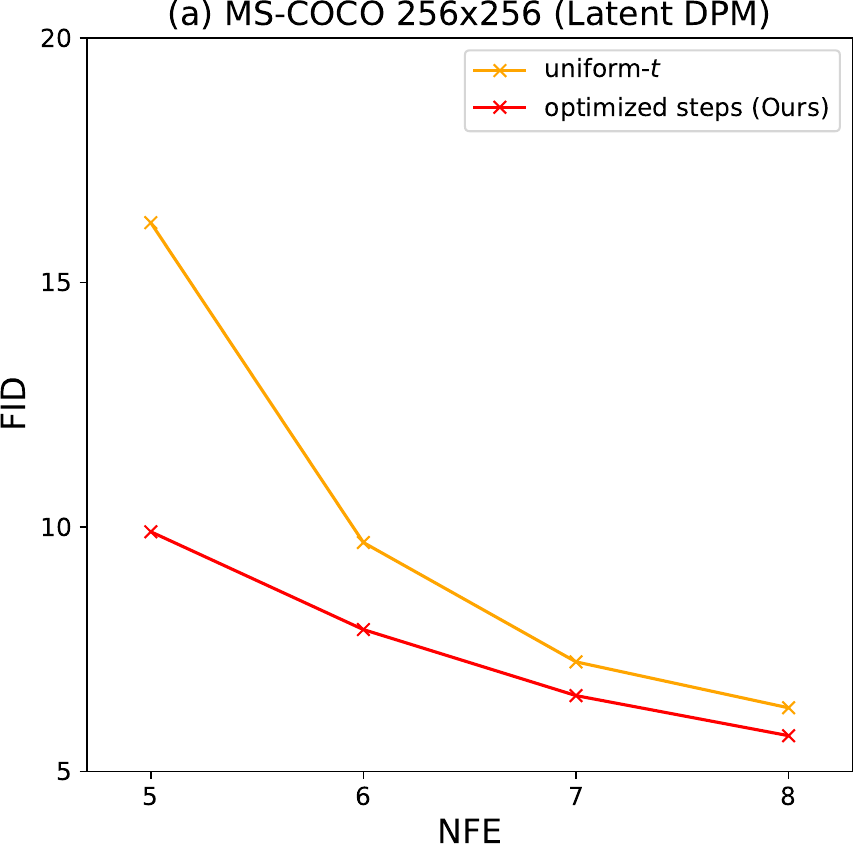}
\caption{Sampling quality measured by FID ($\downarrow$) of different discretization schemes of time steps for UniPC~\citep{zhao2023unipc} with varying NFEs on MS-COCO 256x256 using PixArt-$\alpha$-256 model~\citep{chen2023pixartalpha} (with cfg scale $s=2.5$).}
\label{fig: results_coco}
\end{figure}

\begin{table*}[hbt!]
\centering
\begin{tabular}{lcccccccc}
\toprule
Methods \textbackslash NFEs & 5 & 6 & 7 & 8 & 9 & 10 & 12 & 15\\
\midrule
UniPC with uniform-$\lambda$ & 20.02 & 10.97 & 6.97 & 5.53 &  4.53 & 3.89 & 3.28 & 2.83 \\
\midrule
UniPC with EDM & 26.54 & 15.07 & 11.20 & 11.65 &  10.91 & 8.89 & 5.43 & 3.40 \\
\midrule
UniPC with uniform-$\lambda$-opt & \bf13.66 & \bf8.41 & \bf6.49 & \bf4.84 &  \bf3.82 & \bf3.41 & \bf3.00 & \bf2.82 \\
\bottomrule
\end{tabular}
\caption{Sampling quality measured by FID ($\downarrow$) of different discretization schemes of time steps for UniPC with varying NFEs on FFHQ 64x64 (with $T=80$ and $\epsilon = 0.002$).}
\vspace{-0.05in}
\label{tab:unipc_ffhq64}
\end{table*}

\begin{table*}[hbt!]
\centering
\begin{tabular}{lcccccccc}
\toprule
Methods \textbackslash NFEs & 5 & 6 & 7 & 8 & 9 & 10 & 12 & 15\\
\midrule
UniPC with uniform-$\lambda$ & 12.95 & 8.30 & 5.12 & 4.62 &  4.47 & 3.80 & 2.75 & 2.28 \\
\midrule
UniPC with EDM & 15.83 & 10.30 & 8.46 & 7.83 &  6.78 & 6.38 & 5.25 & 3.09 \\
\midrule
UniPC with uniform-$\lambda$-opt & \bf12.11 & \bf7.49 & \bf5.05 & \bf3.86 &  \bf3.27 & \bf2.74 & \bf2.51 & \bf2.17 \\
\bottomrule
\end{tabular}
\caption{Sampling quality measured by FID ($\downarrow$) of different discretization schemes of time steps for UniPC with varying NFEs on AFHQv2 64x64 (with $T=80$ and $\epsilon = 0.002$).}
\vspace{-0.05in}
\label{tab:unipc_afhq64}
\end{table*}

\begin{table*}[hbt!]
\centering
\begin{tabular}{lcccccccc}
\toprule
Methods \textbackslash NFEs & 5 & 6 & 7 & 8 & 9 & 10 & 12 & 15\\
\midrule
UniPC with uniform-$\lambda$ & 41.14 & 19.81 & 13.01 & 9.83 &  8.31 & 7.01 & 5.30 & 4.00 \\
\midrule
UniPC with uniform-$t$ & 20.28 & 10.47  & 6.57 & 5.13 &  4.46 & 4.14 & 3.75 & 3.45 \\
\midrule
UniPC with uniform-$t$-opt & \bf11.40 & \bf5.95 & \bf4.64 & \bf4.36 &  \bf4.05 & \bf3.81 & \bf3.54 & \bf3.28 \\
\bottomrule
\end{tabular}
\caption{Sampling quality measured by FID ($\downarrow$) of different discretization schemes of time steps for UniPC with varying NFEs on ImageNet 512x512 (with $T=1.0$ and $\epsilon = 0.001$).}
\vspace{-0.05in}
\label{tab:unipc_imagenet512}
\end{table*}

\subsection{Text to Image Generation}
We also evaluate our algorithm on text-to-image generation tasks using the PixArt-$\alpha$ model\footnote{For FID scores, the checkpoint we use is a SAM pretrained model which is then finetuned on COCO.}~\citep{chen2023pixartalpha}. The COCO~\citep{lin2014microsoft} validation set is the standard benchmark for evaluating text-to-image models. We randomly draw 30K prompts from the validation set and report the FID score between model samples generated on these prompts and the reference samples from the full validation set following Imagen~\citep{imagen}. The results are shown in Fig.~\ref{fig: results_coco}. The uniform-$t$ schedule is a widely adopted sampling schedule for latent space text-to-image tasks. In our early experiments, we found the uniform-$t$ schedule significantly outperforms the uniform-$\lambda$ and EDM schedule on latent space text-to-image tasks. Thus we only report the results compared with the uniform-$t$ schedule. Our proposed optimized sampling schedule consistently outperforms the baseline schedules. For example, we improved the FID from $16.22$ to $9.90$ with only $5$NFEs. We also include some samples to demonstrate the effect of our algorithm on text-to-image tasks on Sec.~\ref{sec: additional samples}.



\section{Additional Samples}
\label{sec: additional samples}

We include additional samples with only 5 NFEs in this section. In Fig.~\ref{figure: visual appendix}, we include the samples generated by DiT-XL-2~\citep{peebles2022scalable} on ImageNet 256x256. The samples generated by our method have more details and higher quality. In Fig.~\ref{figure: t2i-pre},\ref{figure: t2i-1},\ref{figure: t2i-2} and \ref{figure: t2i-3}, we include the generated samples corresponding to text prompts generated by PixArt-$\alpha$-512 model~\citep{chen2023pixartalpha}. Our generated samples are clearer and more detailed.

\section{Combining with SciRE-Solver}
\label{sec: SciRE-Solver}

In this section, we provide experimental results for the case of combining our optimized steps with the recently proposed SciRE-Solver~\citep{li2023scire}. See Table~\ref{tab: SciRE-Solver}.

\begin{figure*}[!th]
    \centering
    \setlength{\tabcolsep}{1pt}
    \begin{tabular}{c c c c c}
    \textbf{NFE = 5} & \textbf{Optimized Steps (Ours)} & Uniform-$t$ & EDM & Uniform-$\lambda$\\
     & \textbf{FID = 8.66} & FID = 23.48 & FID = 45.89 & FID = 41.89\\
    \raisebox{5.0\height}{Class label = “ice cream” (928)} &
    \begin{subfigure}{0.16\textwidth}
        \includegraphics[width=\textwidth]{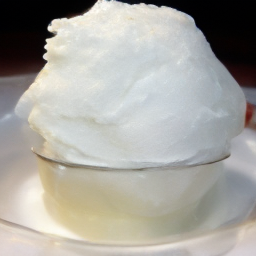}
    \end{subfigure} &
    \begin{subfigure}{0.16\textwidth}
        \includegraphics[width=\textwidth]{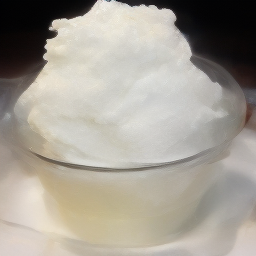}
    \end{subfigure} &
    \begin{subfigure}{0.16\textwidth}
        \includegraphics[width=\textwidth]{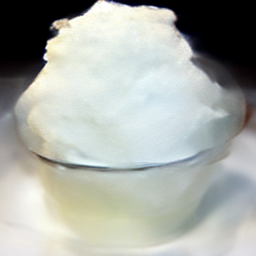}
    \end{subfigure} &
    \begin{subfigure}{0.16\textwidth}
        \includegraphics[width=\textwidth]{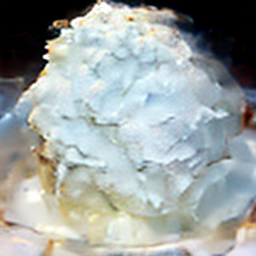}
    \end{subfigure} \\
    \raisebox{5.0\height}{Class label = “panda” (388)} &
    \begin{subfigure}{0.16\textwidth}
        \includegraphics[width=\textwidth]{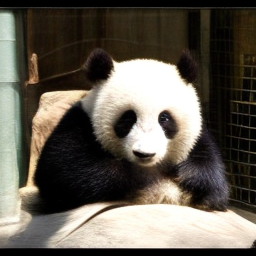}
    \end{subfigure} &
    \begin{subfigure}{0.16\textwidth}
        \includegraphics[width=\textwidth]{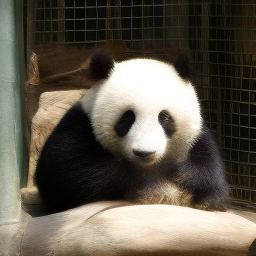}
    \end{subfigure} &
    \begin{subfigure}{0.16\textwidth}
        \includegraphics[width=\textwidth]{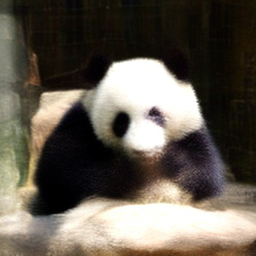}
    \end{subfigure} &
    \begin{subfigure}{0.16\textwidth}
        \includegraphics[width=\textwidth]{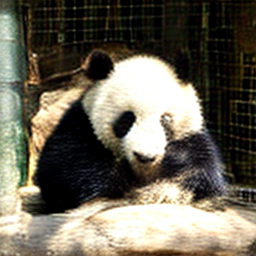}
    \end{subfigure} \\    
    \raisebox{5.0\height}{Class label = “macaw” (88)} &
    \begin{subfigure}{0.16\textwidth}
        \includegraphics[width=\textwidth]{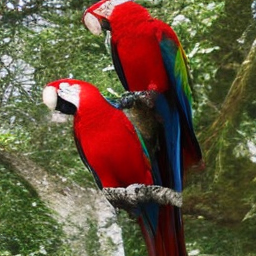}
    \end{subfigure} &
    \begin{subfigure}{0.16\textwidth}
        \includegraphics[width=\textwidth]{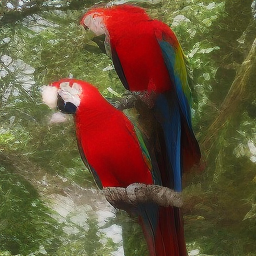}
    \end{subfigure} &
    \begin{subfigure}{0.16\textwidth}
        \includegraphics[width=\textwidth]{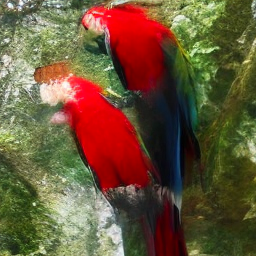}
    \end{subfigure} &
    \begin{subfigure}{0.16\textwidth}
        \includegraphics[width=\textwidth]{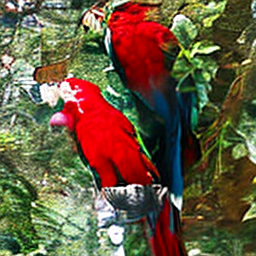}
    \end{subfigure} \\
    \end{tabular}
    \caption{Generated images by UniPC~\citep{zhao2023unipc} with only 5 NFEs for various discretization schemes of time steps from DiT-XL-2 ImageNet 256x256 model~\citep{peebles2022scalable} (with cfg scale $s=1.5$ and the same random seed).}
\label{figure: visual appendix}
\end{figure*}

\begin{figure*}[h]
    \centering
    \setlength{\tabcolsep}{1pt}
    \begin{tabular}{c c c}
    \textbf{NFE = 5} & \textbf{Optimized Steps (Ours)} & Uniform-$t$\\
    \raisebox{2.0\height}{\makecell{\textit{A middle-aged woman of Asian descent,} \\ \textit{her dark hair streaked with silver,} \\ \textit{appears fractured and splintered,} \\ \textit{intricately embedded within}  \\
    \textit{a sea of broken porcelain}}}  &
    \begin{subfigure}{0.26\textwidth}
        \includegraphics[width=\textwidth]{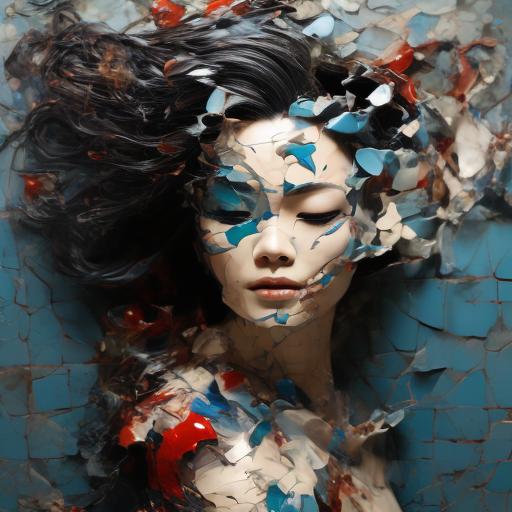}
    \end{subfigure} &
    \begin{subfigure}{0.26\textwidth}
        \includegraphics[width=\textwidth]{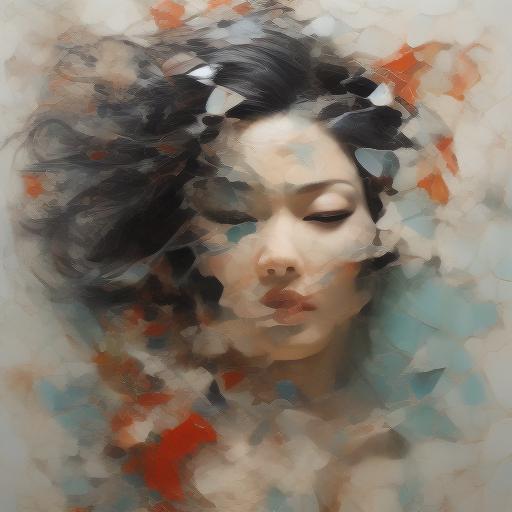}
    \end{subfigure} \\
    \raisebox{8.0\height}{\makecell{\textit{beautiful scene}}} &
    \begin{subfigure}{0.26\textwidth}
        \includegraphics[width=\textwidth]{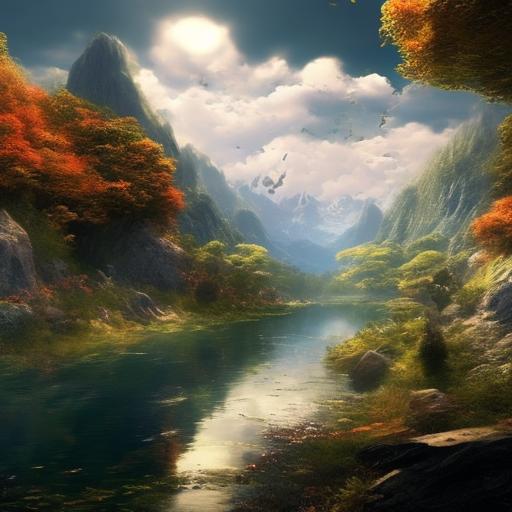}
    \end{subfigure} &
    \begin{subfigure}{0.26\textwidth}
        \includegraphics[width=\textwidth]{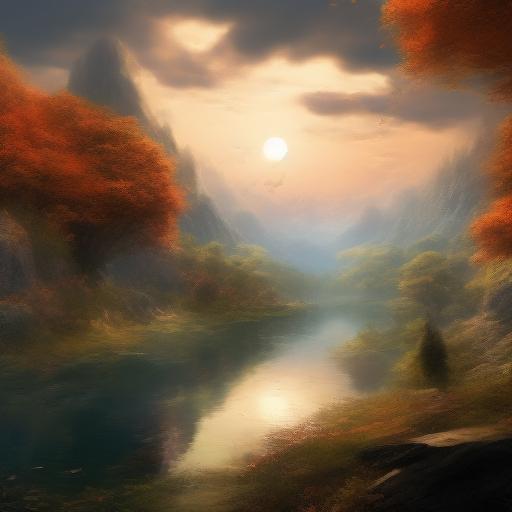}
    \end{subfigure} \\
    \end{tabular}
    \caption{Generated images by UniPC~\citep{zhao2023unipc} with only 5 NFEs for various discretization schemes of time steps from PixArt-$\alpha$-512 model~\citep{chen2023pixartalpha} (with cfg scale $s=2.5$ and the same random seed).}
\label{figure: t2i-pre}
\end{figure*}

\begin{figure*}[h]
    \centering
    \setlength{\tabcolsep}{1pt}
    \begin{tabular}{c c c}
    \textbf{NFE = 5} & \textbf{Optimized Steps (Ours)} & Uniform-$t$\\
    \raisebox{6.0\height}{\makecell{\textit{A alpaca made of colorful} \\ \textit{building blocks, cyberpunk}}} &
    \begin{subfigure}{0.38\textwidth}
        \includegraphics[width=\textwidth]{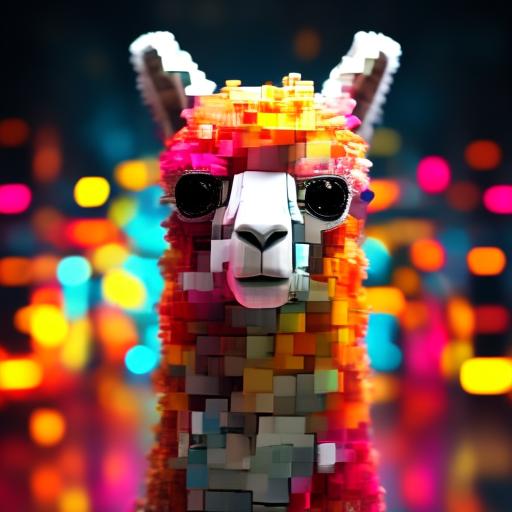}
    \end{subfigure} &
    \begin{subfigure}{0.38\textwidth}
        \includegraphics[width=\textwidth]{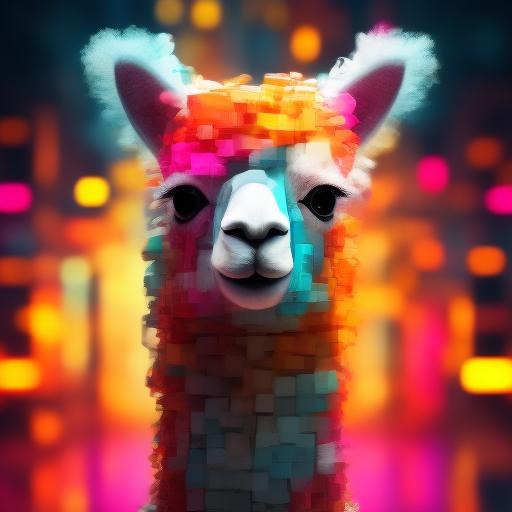}
    \end{subfigure} \\
    \raisebox{10.0\height}{\makecell{\textit{bird's eye view of a city}}} &
    \begin{subfigure}{0.38\textwidth}
        \includegraphics[width=\textwidth]{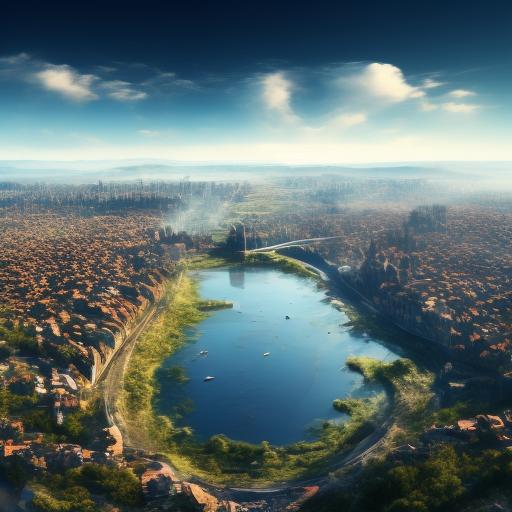}
    \end{subfigure} &
    \begin{subfigure}{0.38\textwidth}
        \includegraphics[width=\textwidth]{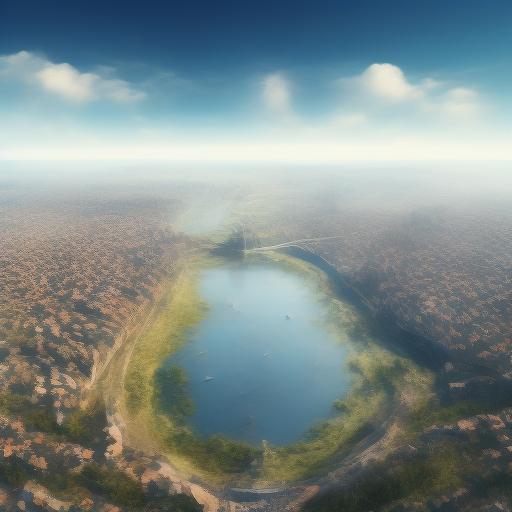}
    \end{subfigure} \\
    \raisebox{4.0\height}{\makecell{\textit{A worker that looks like} \\ \textit{a mixture of cow and horse} \\ \textit{is working hard to type code} }} &
    \begin{subfigure}{0.38\textwidth}
        \includegraphics[width=\textwidth]{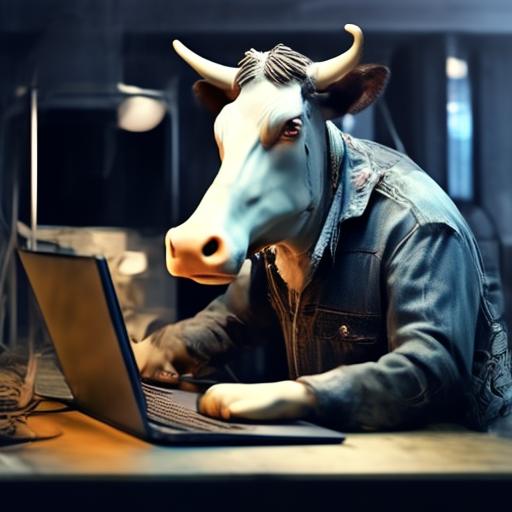}
    \end{subfigure} &
    \begin{subfigure}{0.38\textwidth}
        \includegraphics[width=\textwidth]{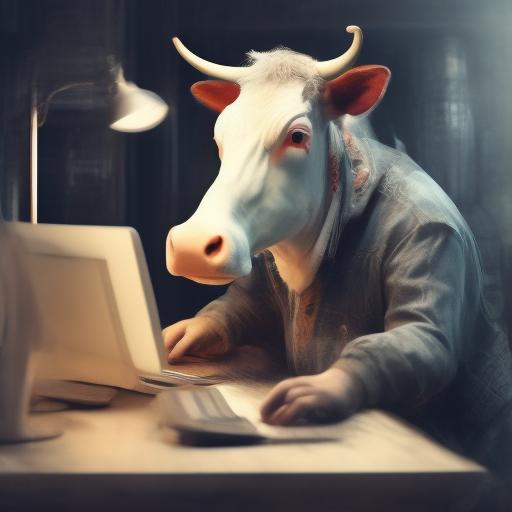}
    \end{subfigure} \\
    \end{tabular}
    \caption{Generated images by UniPC~\citep{zhao2023unipc} with only 5 NFEs for various discretization schemes of time steps from PixArt-$\alpha$-512 model~\citep{chen2023pixartalpha} (with cfg scale $s=2.5$ and the same random seed).}
\label{figure: t2i-1}
\end{figure*}

\begin{figure*}[h]
    \centering
    \setlength{\tabcolsep}{1pt}
    \begin{tabular}{c c c}
    \textbf{NFE = 5} & \textbf{Optimized Steps (Ours)} & Uniform-$t$\\
    \raisebox{6.0\height}{\makecell{\textit{A transparent sculpture of} \\ \textit{a duck made out of glass}}} &
    \begin{subfigure}{0.38\textwidth}
        \includegraphics[width=\textwidth]{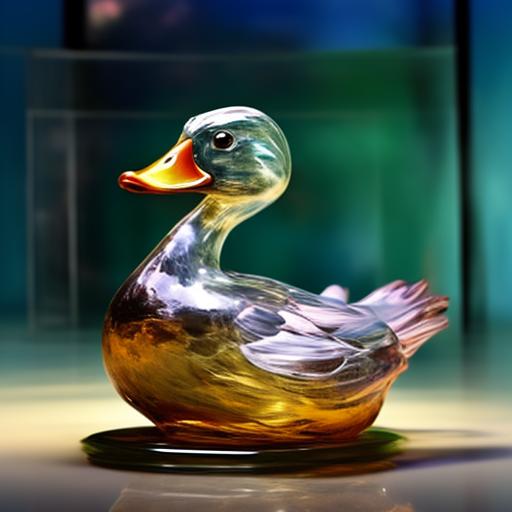}
    \end{subfigure} &
    \begin{subfigure}{0.38\textwidth}
        \includegraphics[width=\textwidth]{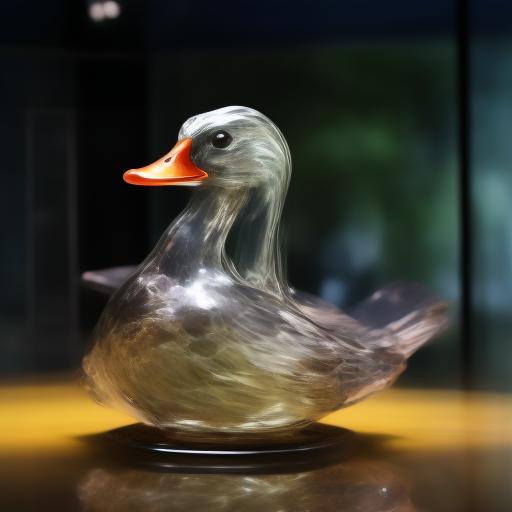}
    \end{subfigure} \\
    \raisebox{3.0\height}{\makecell{\textit{An illustration of a human} \\ \textit{heart made of translucent} \\ \textit{glass, standing on a} \\ \textit{pedestal amidst a stormy sea} }} &
    \begin{subfigure}{0.38\textwidth}
        \includegraphics[width=\textwidth]{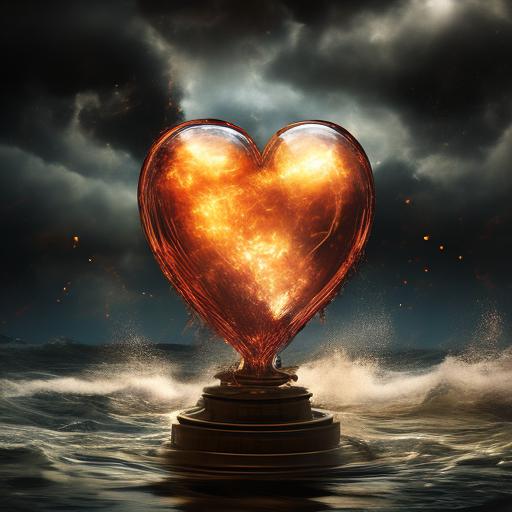}
    \end{subfigure} &
    \begin{subfigure}{0.38\textwidth}
        \includegraphics[width=\textwidth]{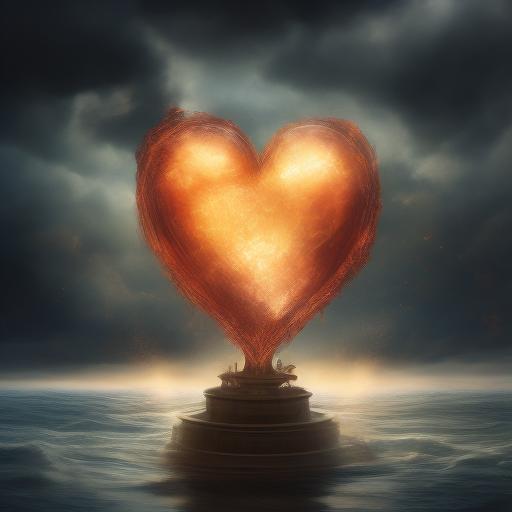}
    \end{subfigure} \\
    \raisebox{6.0\height}{\makecell{\textit{A boy and a girl} \\ \textit{fall in love} }} &
    \begin{subfigure}{0.38\textwidth}
        \includegraphics[width=\textwidth]{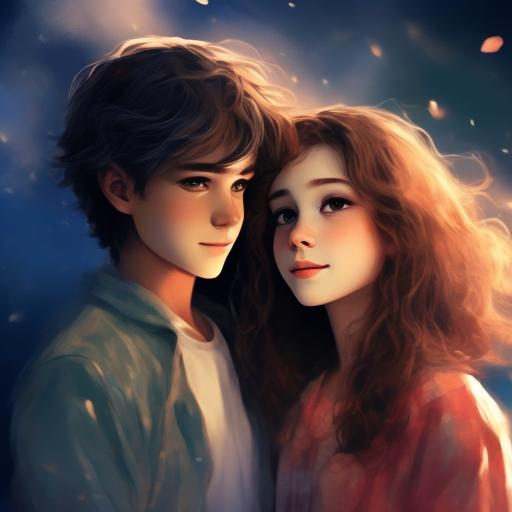}
    \end{subfigure} &
    \begin{subfigure}{0.38\textwidth}
        \includegraphics[width=\textwidth]{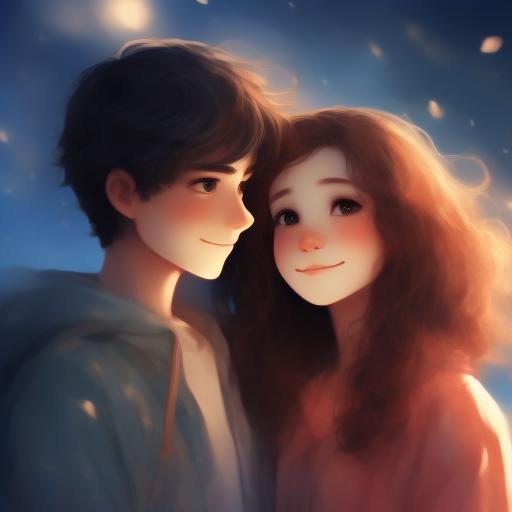}
    \end{subfigure} \\
    \end{tabular}
    \caption{Generated images by UniPC~\citep{zhao2023unipc} with only 5 NFEs for various discretization schemes of time steps from PixArt-$\alpha$-512 model~\citep{chen2023pixartalpha} (with cfg scale $s=2.5$ and the same random seed).}
\label{figure: t2i-2}
\end{figure*}

\begin{figure*}[h]
    \centering
    \setlength{\tabcolsep}{1pt}
    \begin{tabular}{c c c}
    \textbf{NFE = 5} & \textbf{Optimized Steps (Ours)} & Uniform-$t$\\
    \raisebox{6.0\height}{\makecell{\textit{Luffy from ONEPIECE,} \\ \textit{handsome face, fantasy}}} &
    \begin{subfigure}{0.38\textwidth}
        \includegraphics[width=\textwidth]{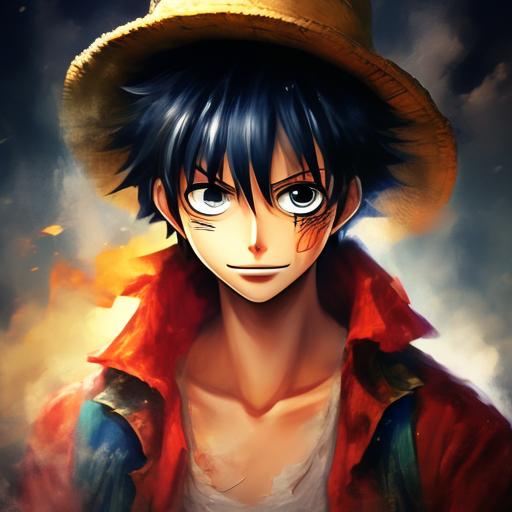}
    \end{subfigure} &
    \begin{subfigure}{0.38\textwidth}
        \includegraphics[width=\textwidth]{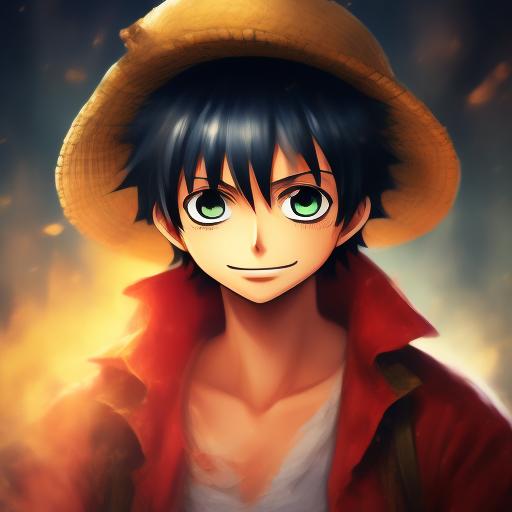}
    \end{subfigure} \\
    \raisebox{4.0\height}{\makecell{\textit{A 2D animation of a folk} \\ \textit{music band composed of} \\ \textit{anthropomorphic autumn leaves} }} &
    \begin{subfigure}{0.38\textwidth}
        \includegraphics[width=\textwidth]{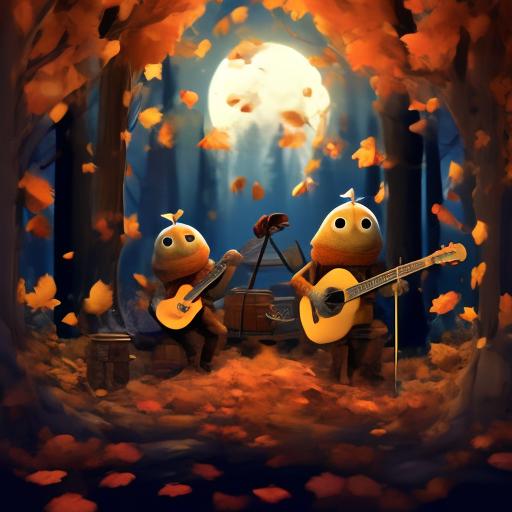}
    \end{subfigure} &
    \begin{subfigure}{0.38\textwidth}
        \includegraphics[width=\textwidth]{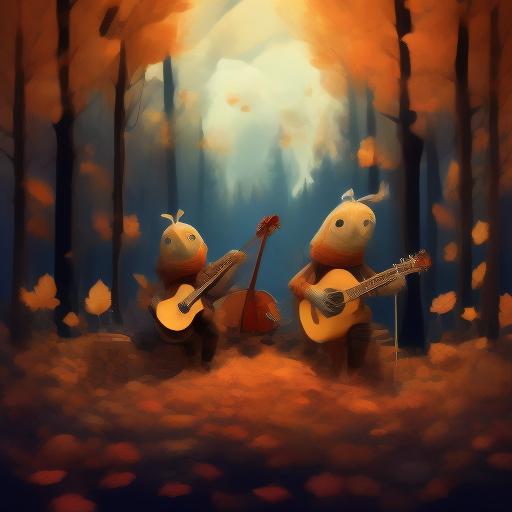}
    \end{subfigure} \\
    \raisebox{3.0\height}{\makecell{\textit{A surreal parallel world} \\ \textit{where mankind avoid extinction} \\ \textit{by preserving nature,} \\ \textit{epic trees, water streams} }} &
    \begin{subfigure}{0.38\textwidth}
        \includegraphics[width=\textwidth]{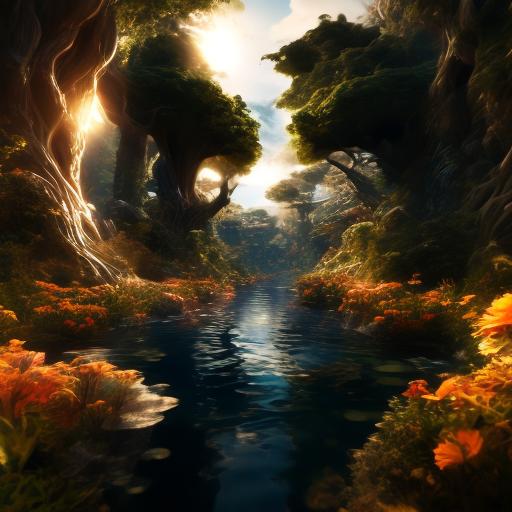}
    \end{subfigure} &
    \begin{subfigure}{0.38\textwidth}
        \includegraphics[width=\textwidth]{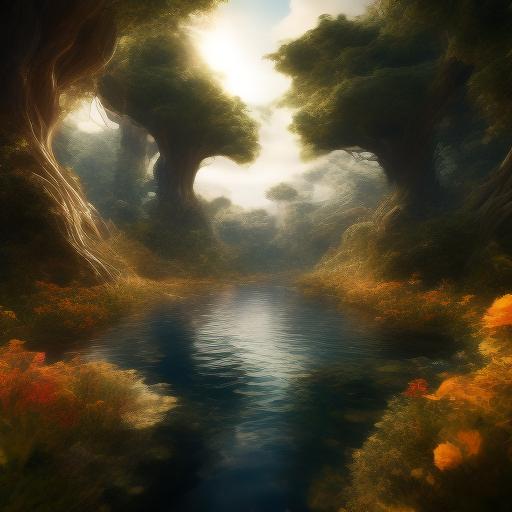}
    \end{subfigure} \\
    \end{tabular}
    \caption{Generated images by UniPC~\citep{zhao2023unipc} with only 5 NFEs for various discretization schemes of time steps from PixArt-$\alpha$-512 model~\citep{chen2023pixartalpha} (with cfg scale $s=2.5$ and the same random seed).}
\label{figure: t2i-3}
\end{figure*}

\begin{table*}[hbt!]
\centering
\begin{tabular}{lcccc}
\toprule
Methods \textbackslash NFEs & 5 & 6 & 7 & 8 \\
\midrule
Uniform-$t$ & 36.18 & 15.93 & 7.57 & 4.48 \\
Uniform-$\lambda$ & 92.80  & 41.19 & 21.08 & 11.25\\
EDM & 90.48& 47.18& 23.32 & 11.10\\
Optimized steps (Ours) & \bf16.07& \bf6.31& \bf4.12& \bf3.66\\
\bottomrule
\end{tabular}
\caption{Sampling quality measured by FID ($\downarrow$) of different discretization schemes of time steps for SciRE-Solver with varying NFEs on ImageNet 256x256 (using the DiT-XL-2 model with cfg = 1.5).}
\vspace{-0.05in}
\label{tab: SciRE-Solver}
\end{table*}